\documentclass{article}

\usepackage[preprint]{neurips_2025}

\usepackage[utf8]{inputenc} 
\usepackage[T1]{fontenc}    
\usepackage{hyperref}       
\usepackage{url}            
\usepackage{booktabs}       
\usepackage{amsfonts}       
\usepackage{nicefrac}       
\usepackage{microtype}      
\usepackage{xcolor}         

\usepackage{alltt} 

\usepackage[toc,page]{appendix}
\usepackage{titletoc}

\usepackage{amsmath}
\usepackage{amsthm}
\usepackage{amssymb}
\usepackage{tikz}
\usepackage{array}
\usepackage{colortbl}

\newtheorem{proposition}{Proposition}
\newtheorem*{proposition*}{Proposition}
\usepackage{wrapfig}
\usepackage{graphicx}
\usepackage{xspace}
\usepackage{multirow}
\usepackage{makecell}
\usepackage{nicematrix}
\usepackage[normalem]{ulem}
\usepackage{listings}
\usepackage{comment}
\usepackage{subfigure}
\usepackage{caption}
\usepackage{rotate}
\usepackage{float}
\usepackage{cleveref}
\usepackage{algorithmic}
\usepackage{algorithm}
\usepackage[bottom]{footmisc}
\usepackage{tcolorbox}
\usepackage{tabularx}
\usepackage{arydshln}

\NewDocumentCommand{\rot}{O{45} O{1em} m}{\makebox[#2][l]{\rotatebox{#1}{#3}}}%

\definecolor{babypink}{rgb}{0.96, 0.76, 0.76}
\definecolor{coralpink}{rgb}{0.97, 0.51, 0.47}
\definecolor{aqua}{rgb}{0.0, 1.0, 1.0}
\definecolor{columbiablue}{rgb}{0.61, 0.87, 1.0}
\definecolor{cyan}{RGB}{222, 255, 255}
\definecolor{turquoiseblue}{rgb}{0.0, 1.0, 0.94}
\definecolor{realcyan}{RGB}{0, 200, 200}

\hypersetup{
    colorlinks = true,
    linkbordercolor = {white},
    linkcolor = blue!60!black,
    citecolor = blue!60!black,
    urlcolor = blue!60!black
}

\title{Understanding LLM Evaluator Behavior:
A Structured Multi‑Evaluator Framework for Merchant Risk Assessment}

\author{
    \textbf{Liang Wang} \quad
    \textbf{Junpeng Wang} \quad
    \textbf{Chin-Chia Michael Yeh} \quad
    \textbf{Yan Zheng} \\
    \textbf{Jiarui Sun} \quad
    \textbf{Xiran Fan} \quad 
    \textbf{Xin Dai} \quad
    \textbf{Yujie Fan} \\
    \textbf{Yiwei Cai}
    \vspace{0.4em} \\
    \textnormal{Visa Research} \\ 
    \textnormal{900 Metro Center Blvd} \\
    \textnormal {Foster City, CA 94404} \\
    \texttt{(liawang,junpenwa,miyeh,yazheng,jiaruis2,xiyafan,xidai,yufan,yicai)}@visa.com
}

\begin{document}

\maketitle
\begin{abstract}
Large Language Models (LLMs) are increasingly used as evaluators of reasoning
quality, yet their reliability and bias in \emph{payments‑risk} settings remain
poorly understood. We introduce a \textbf{structured multi‑evaluator framework}
for assessing LLM reasoning in Merchant Category Code (MCC)‑based merchant risk
assessment, combining a five‑criterion domain rubric with Monte‑Carlo scoring
to evaluate both rationale quality and evaluator stability. Five frontier LLMs
(GPT‑5.1, Gemini‑2.5 Pro, Grok~4, Claude~4.5 Sonnet, Perplexity Sonar) generate
and cross‑evaluate MCC risk rationales under \emph{attributed} and
\emph{anonymized} conditions. To establish a principled, judge‑independent
reference, we introduce a \textbf{consensus‑deviation metric} that eliminates
circularity by comparing each judge’s score to the mean of all \emph{other}
judges, yielding a theoretically grounded measure of self‑evaluation and cross‑model
deviation. Our results reveal substantial heterogeneity in evaluator behavior:  GPT‑5.1 and Claude~4.5 Sonnet show negative self‑evaluation bias ($-0.33$, $-0.31$), while Gemini‑2.5 Pro and Grok~4 display strong positive bias ($+0.77$, $+0.71$), with
bias direction persisting but attenuating by 25.8\% under anonymization.
Evaluation by 26 payment‑industry experts shows that LLM judges assign scores
averaging +0.46 points higher than human consensus, and that the negative bias
of GPT‑5.1 and Claude~4.5 Sonnet relative to LLM peers reflects \emph{closer
alignment with human judgment}. Ground‑truth validation using payment‑network
transaction data shows that four models—Claude~4.5 Sonnet, Gemini‑2.5 Pro,
Grok~4, and GPT‑5.1—exhibit statistically significant alignment (Spearman
$\rho = 0.56$–$0.77$), confirming that the evaluation framework captures genuine quality. Overall, the framework offers a replicable basis for evaluating LLM‑as‑a‑judge systems in payment‑risk workflows and highlights the need for rigorous, bias‑aware protocols when deploying LLM evaluators in operationally sensitive
financial settings.
\end{abstract}

\section{Introduction}
\label{sec:introduction}

Large Language Models (LLMs) are increasingly deployed not only as content generators but also as evaluators, judging the quality of text, code, and reasoning produced by other models ~\cite{chang2024survey,chen2025llm,kahng2024llm,li2025generation,li2024llms,tan2024large,zhang2024llmjudgesurvey,zheng2023judging}. This LLM-as-a-judge paradigm raises foundational questions about reliability, stability, and bias—questions that are particularly salient in \emph{payments-risk} domains, where evaluation errors can affect fraud analytics, merchant onboarding, regulatory compliance, and transaction
monitoring. Merchant Category Code (MCC) risk assessment provides a demanding testbed: with more than 800 heterogeneous categories, coherent rationale generation requires integrating business-model stability, regulatory exposure, fraud typologies, return and refund behavior, and chargeback dynamics.

Despite growing interest in LLM-as-a-judge systems, no prior work provides a structured, domain-aligned evaluation of LLM reasoning in payments-risk settings. Existing benchmarks emphasize general linguistic competence or broad reasoning skills, but do not examine whether models can produce correct,
complete, and operationally grounded MCC risk rationales aligned with industry practice. Moreover, prior studies rarely assess the stability of LLM-generated evaluations across repeated stochastic sampling, leaving open how consistent a given evaluator is across runs. A further limitation is the treatment of self-evaluation bias: most existing work assumes positive self-preference and uses pairwise tests that cannot measure bias magnitude, detect negative self-critique, or disentangle judge-specific bias from true quality differences~\cite{balog2025rankers,dietz2025llm,jiang2025artificial,panickssery2024llm, wang2024large,wataoka2024self,xu2402pride}. These gaps underscore the need for a principled methodology that evaluates both the \emph{quality} of LLM reasoning and the \emph{behavior} of LLM evaluators in high-stakes, domain-specific contexts.

To address these challenges, we introduce a \textbf{structured multi-evaluator framework} for MCC-based merchant-risk reasoning. The framework integrates a five-criterion rubric (Accuracy, Rationale Quality, Consistency, Completeness,
Practical Applicability) with a Monte Carlo evaluation procedure in which each model performs multiple independently sampled scoring runs, yielding estimates of evaluator stability ($\mu \pm \sigma$). We evaluate five frontier LLMs—GPT-5.1,
Gemini-2.5 Pro, Grok~4, Claude~4.5 Sonnet, and Perplexity Sonar—both as rationale generators and as evaluators. This design enables analysis not only of reasoning quality but also of how consistently different judges evaluate complex merchant categories under stochastic variability.

A central contribution of this work is a \textbf{consensus-deviation metric} that provides a principled method for quantifying both self-evaluation and cross-model bias. The metric eliminates circularity by comparing each judge's
score only to the mean assigned by all \emph{other} judges, ensuring that the reference standard remains independent of the judge being evaluated. This makes it possible to isolate judge-specific tendencies, measure bias magnitude robustly, and detect both positive (self-promoting) and negative
(self-critical) forms of self-evaluation. Using this metric, we conduct cross-evaluation experiments under two conditions—\emph{attributed} (source model disclosed) and \emph{anonymized} (source concealed)—allowing us to
distinguish biases driven by authorship recognition from those reflecting deeper evaluative tendencies.

Our results show substantial heterogeneity in evaluator behavior. GPT-5.1 and Claude~4.5 Sonnet exhibit \emph{negative} self-evaluation bias ($-0.33$ and $-0.31$
points), consistently scoring their own outputs below peer consensus—a behavior not identified in prior LLM self-evaluation work. In contrast, Gemini-2.5 Pro and Grok~4 display strong \emph{positive} bias ($+0.77$ and $+0.71$), while Perplexity Sonar exhibits modest positive bias ($+0.21$). Bias direction persists under anonymization: although anonymization reduces magnitude by 25.8\% on average, it does not reverse direction, indicating that evaluator tendencies reflect underlying
model characteristics rather than explicit authorship cues. These findings demonstrate that evaluation behavior varies systematically across models and can differ sharply from traditional expectations of universal self-preference.

To validate these LLM-based findings, we conducted complementary human expert evaluation: 26 domain experts from the payments industry—including research scientists and experienced business partners with deep expertise in merchant risk assessment, fraud prevention, and payment operations—independently evaluated the same LLM-generated rationales using the identical evaluation rubric. This human validation reveals that LLM judges systematically assign scores averaging +0.46 points higher than human expert consensus, with models exhibiting negative bias relative to LLM peers (GPT-5.1, Claude-4.5 Sonnet) demonstrating closest alignment with human judgment. We further validate findings against four years of payment network transaction data, showing that top-rated models achieve Spearman $\rho = 0.56$--0.77 with empirical merchant risk patterns. This triangulated validation—combining peer consensus, human expert assessment, and empirical ground truth—confirms that the evaluation framework captures genuine quality differences rather than shared model artifacts.

\paragraph{Contributions.}
\begin{enumerate}
    \item We provide the first structured, domain-aligned evaluation of LLM reasoning in MCC-based \emph{payments-risk} settings, establishing a transparent and replicable foundation for assessing LLM-as-a-judge systems.
    \item We develop a Monte Carlo evaluation framework that quantifies scoring stability through repeated independent assessments, revealing substantial cross-model differences in evaluator consistency.
    \item We introduce a consensus-deviation metric with mathematically proven
    circularity prevention, enabling rigorous measurement of both positive and negative self-evaluation bias.
    \item We present the first quantitative evidence of negative
    self-evaluation bias in frontier LLMs, extending and challenging prior work on self-preference.
    \item We show that bias direction persists under anonymization and that anonymization reduces magnitude without eliminating directional tendencies.
    \item We validate findings through complementary human expert evaluation and empirical ground-truth analysis, demonstrating that models exhibiting conservative scoring relative to LLM peers align more closely with human judgment and, for top performers, with payment network transaction data.
\end{enumerate}

\paragraph{Paper organization.}
Section~\ref{sec:mcc_generation} describes MCC risk rationale generation.
Section~\ref{sec:evaluation} presents the Monte Carlo evaluation framework including human expert validation.
Section~\ref{sec:bias_metrics} introduces the consensus-deviation metric, characterizes bias patterns, and compares against human baseline.
Section~\ref{sec:empirical_validation} validates findings against payment network transaction data.
Section~\ref{sec:related_work} reviews related literature.
Section~\ref{sec:discussions} discusses implications and limitations.
Section~\ref{sec:conclusion} concludes with implications for trustworthy LLM-as-a-judge deployment.

\section{MCC Risk Rationale Generation}
\label{sec:mcc_generation}

Figure~\ref{fig:prompt_structure} illustrates the prompt design used to elicit
structured MCC risk rationales that serve as evaluation targets in our
multi-evaluator framework. This section describes the prompt structure, the
five core risk dimensions, and the LLM-generated outputs analyzed in later
sections.


\begin{figure}[t]
\centering
\resizebox{\columnwidth}{!}{%
\begin{tikzpicture}[
    every node/.style={font=\sffamily},
    header/.style={fill=blue!10, draw=blue!60, line width=1pt, rounded corners=3pt, minimum height=1.0cm},
    mainbox/.style={draw=gray!50, line width=1.5pt, rounded corners=5pt},
    badge/.style={circle, minimum size=0.55cm, line width=1.5pt, draw=white},
    sectiontitle/.style={font=\sffamily\bfseries\normalsize, color=blue!80},
    bodytext/.style={font=\sffamily\small},
    jsonbox/.style={fill=gray!5, draw=gray!40, line width=1pt, rounded corners=3pt}
]

\draw[mainbox] (0,2.1) rectangle (18,8.5);

\draw[gray!50, line width=1.5pt] (6,2.3) -- (6,8.2);
\draw[gray!50, line width=1.5pt] (12,2.3) -- (12,8.2);

\node[header, minimum width=5cm] at (3,7.9) {};
\node[font=\sffamily\bfseries\normalsize, color=blue!80] at (3,8.15) {INPUT};
\node[font=\sffamily\small] at (3,7.7) {Risk Levels};

\node[badge, fill=green!60] at (0.8,6.5) {\small\textbf{\textcolor{white}{1}}};
\node[bodytext, anchor=west] at (1.5,6.5) {Very Low Risk};

\node[badge, fill=lime!70] at (0.8,5.6) {\small\textbf{\textcolor{white}{2}}};
\node[bodytext, anchor=west] at (1.5,5.6) {Low Risk};

\node[badge, fill=yellow!70!orange] at (0.8,4.7) {\small\textbf{\textcolor{white}{3}}};
\node[bodytext, anchor=west] at (1.5,4.7) {Medium Risk};

\node[badge, fill=orange!80] at (0.8,3.8) {\small\textbf{\textcolor{white}{4}}};
\node[bodytext, anchor=west] at (1.5,3.8) {High Risk};

\node[badge, fill=red!70] at (0.8,2.9) {\small\textbf{\textcolor{white}{5}}};
\node[bodytext, anchor=west] at (1.5,2.9) {Very High Risk};

\node[header, minimum width=5cm] at (9,7.9) {};
\node[font=\sffamily\bfseries\normalsize, color=blue!80] at (9,8.15) {RATIONALE};
\node[font=\sffamily\small] at (9,7.7) {Instructions};

\node[sectiontitle, anchor=west] at (6.5,6.8) {Explicitly Address:};
\node[bodytext, anchor=west] at (6.8,6.35) {$\bullet$ Business Model Stability};
\node[bodytext, anchor=west] at (6.8,5.95) {$\bullet$ Regulatory Exposure};
\node[bodytext, anchor=west] at (6.8,5.55) {$\bullet$ Fraud Exposure};
\node[bodytext, anchor=west] at (6.8,5.15) {$\bullet$ Return Patterns};
\node[bodytext, anchor=west] at (6.8,4.75) {$\bullet$ Chargeback Activity};

\node[sectiontitle, anchor=west] at (6.5,4.2) {Include:};
\node[bodytext, anchor=west, color=green!60!black] at (6.8,3.8) {$\checkmark$ 3 Representative MCCs};

\node[sectiontitle, anchor=west] at (6.5,3.15) {Constraint:};
\node[bodytext, anchor=west, color=red!70!black] at (6.8,2.75) {$\times$ No Numerical Metrics};
\node[bodytext, anchor=west, color=red!70!black] at (6.8,2.35) {$\times$ No Industry Jargon};

\node[header, minimum width=5cm] at (15,7.9) {};
\node[font=\sffamily\bfseries\normalsize, color=blue!80] at (15,8.15) {OUTPUT};
\node[font=\sffamily\small] at (15,7.7) {Format};

\node[sectiontitle, anchor=west] at (12.5,6.8) {JSON Array};
\node[font=\sffamily\small, anchor=west] at (12.5,6.3) {with objects containing:};

\node[jsonbox, minimum width=4.5cm, minimum height=3.3cm, anchor=north west] at (12.5,5.9) {};
\node[font=\ttfamily\footnotesize, anchor=north west, align=left, inner sep=0pt, node distance=0pt] at (12.7,5.75) {
    [\\[0.05cm]
    \hspace{0.3cm}\{\\[0.05cm]
    \hspace{0.6cm}"risk\_level\_\\[0.05cm]
    \hspace{0.7cm}definition": ...\\[0.1cm]
    \hspace{0.6cm}"rationale": ...\\[0.05cm]
    \hspace{0.3cm}\},\\[0.05cm]
    \hspace{0.3cm}...\\[0.05cm]
    ]
};

\end{tikzpicture}%
}
\caption{\textbf{Structure of the MCC Risk Rationale Prompt.} \textbf{Left}: INPUT specifies five risk levels from very low to very high risk. \textbf{Center:} RATIONALE instructions require explicit coverage of five payments-risk dimensions (Business Model Stability, Regulatory Exposure, Fraud Exposure, Return Patterns, Chargeback Activity) along with 3 representative MCCs, while prohibiting numerical metrics and industry jargon.\textbf{Right:} OUTPUT format specifies a structured JSON array containing risk level definitions and rationales. Full prompt text appears in Appendix~\ref{app:prompt}.}
\label{fig:prompt_structure}
\end{figure}
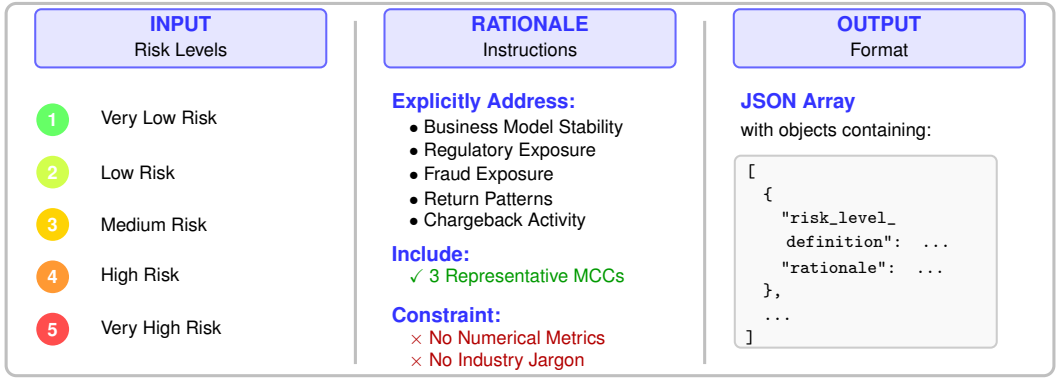

\subsection{Prompt Design and Risk Level Specification}

Five frontier LLMs—GPT‑5.1, Gemini‑2.5 Pro, Grok~4, Claude‑4.5 Sonnet, and
Perplexity Sonar—were each instructed to act as global payment-risk experts and
produce qualitative assessments across a five-level spectrum (Very Low to Very
High Risk). To ensure comparability across models, the prompt requires each
rationale to address the same five domain-relevant dimensions:

\begin{itemize}
    \item Business Model Stability
    \item Regulatory Exposure
    \item Fraud Exposure
    \item Return / Refund Patterns
    \item Chargeback Activity
\end{itemize}

For each risk level, models must also select three representative MCCs from the over 800 distinct merchant categories. To promote domain-aligned, interpretable reasoning, the prompt restricts numerical claims and discourages specialized industry jargon, and the output must follow a structured JSON schema enabling machine-readable analysis.

Figure~\ref{fig:selected_mccs} provides examples of representative MCCs spanning diverse business models and risk profiles. These categories anchor the
rationales in concrete merchant types that reflect realistic payments-risk
considerations.



\begin{figure}[t]
\centering
\small
\begin{tabular}{>{\centering\arraybackslash}p{2.5cm}>{\footnotesize}p{8cm}}
\toprule
\textbf{MCC Code} & \textbf{Merchant Category Description} \\
\midrule
\rowcolor{blue!5}
5411 & GROCERY STORES/SUPERMARKETS \\
\rowcolor{white}
5541 & SERVICE STATIONS \\
\rowcolor{blue!5}
5552 & ELECTRIC VEHICLE CHARGING \\
\rowcolor{white}
5651 & FAMILY CLOTHING STORES \\
\rowcolor{blue!5}
5732 & ELECTRONICS STORES \\
\rowcolor{white}
5812 & RESTAURANTS \\
\rowcolor{blue!5}
5816 & DIGITAL GOODS: GAMES \\
\rowcolor{white}
5967 & INBOUND TELEMARKETING \\
\rowcolor{blue!5}
6051 & QUASI-CASH \\
\rowcolor{white}
7273 & DATING \& ESCORT SERVICES \\
\bottomrule
\end{tabular}
\caption{\textbf{Representative MCCs}. Ten MCCs spanning diverse risk levels and business models: from low-risk essential services (grocery stores, service stations) to high-risk categories (quasi-cash, telemarketing, dating services). Full descriptions appear in the Visa Merchant Data Standards Manual~\cite{visa2023merchant} and Mastercard Quick Reference Booklet~\cite{mastercard2023merchant}.}
\label{fig:selected_mccs}
\end{figure}

\subsubsection{Design Principle and Data Access}

The five‑level structure yields concise prompts and supports interpretable reasoning aligned with payments‑risk practice. Importantly, all LLMs had access only to a public MCC‑to‑Name mapping table; they received no transaction‑level, merchant‑level, or proprietary network data. As a result, all generated rationales reflect general knowledge learned by these LLMs during pretraining rather than relying on any proprietary payment‑transaction information.

\subsection{LLM‑Generated Risk Rationales}

Each model produces structured JSON outputs containing five rationales—one per
risk level—covering all required dimensions and corresponding representative
MCCs. These outputs constitute the core artifacts evaluated for reasoning
quality, cross-model agreement, bias, and stability in Section~\ref{sec:evaluation}.

Figure~\ref{fig:llm_rationales} shows example rationales generated by
Claude‑4.5 Sonnet. Each card reflects the five required dimensions, a coherent
shift in language intensity across risk levels, and representative MCCs aligned with industry expectations (e.g., essential services at lower risk levels and high-volatility or high-chargeback sectors at higher levels), despite models’ access only to category names.


\begin{figure*}[t]
\centering
\resizebox{\textwidth}{!}{%
\begin{tikzpicture}[
    every node/.style={font=\small},
    card/.style={draw, line width=1.5pt, rounded corners=5pt, minimum width=5.5cm, minimum height=6.5cm},
    badge/.style={rounded corners=3pt, minimum width=5cm, minimum height=0.6cm},
    mccheader/.style={font=\scriptsize\bfseries}
]

\node[card, fill=green!10, draw=green!60] (card1) at (0,6.5) {};
\node[badge, fill=green!70!black, anchor=north] at (card1.north) [yshift=-0.2cm] {
    \textbf{\textcolor{white}{1 — Very Low Risk}}
};
\node[anchor=north west, text width=5cm, align=justify, font=\footnotesize] at (card1.north west) [xshift=0.25cm, yshift=-1.0cm] {
    Merchants operate highly stable, predictable business models with minimal seasonal variation and long operational histories. Low regulatory complexity with straightforward compliance. Fraud exposure is minimal due to transparent transaction patterns. Returns are rare. Chargeback activity is virtually non-existent.
};
\node[mccheader, color=green!70!black, anchor=south west] at (card1.south west) [xshift=0.25cm, yshift=1.2cm] {Representative MCCs};
\draw[green!60, line width=0.5pt] ([xshift=0.25cm, yshift=1.13cm]card1.south west) -- ([xshift=-0.25cm, yshift=1.13cm]card1.south east);
\node[anchor=south west, text width=5cm, align=left, font=\scriptsize\bfseries] at (card1.south west) [xshift=0.35cm, yshift=0.2cm] {
    • 5411 — Grocery Stores\\
    • 5912 — Drug Stores \& Pharmacies\\
    • 5541 — Service Stations
};

\node[card, fill=lime!5, draw=lime!70] (card2) at (6,6.5) {};
\node[badge, fill=lime!70!black, anchor=north] at (card2.north) [yshift=-0.2cm] {
    \textbf{\textcolor{white}{2 — Low Risk}}
};
\node[anchor=north west, text width=5cm, align=justify, font=\footnotesize] at (card2.north west) [xshift=0.25cm, yshift=-1.0cm] {
    Merchants exhibit stable business models with moderate predictability and occasional seasonal fluctuations that are well-understood and manageable. Regulatory exposure is routine with standard licensing. Fraud exposure is occasional and low-value. Return patterns are predictable and align with industry norms. Chargeback activity is low and manageable.
};
\node[mccheader, color=lime!70!black, anchor=south west] at (card2.south west) [xshift=0.25cm, yshift=1.2cm] {Representative MCCs};
\draw[lime!70, line width=0.5pt] ([xshift=0.25cm, yshift=1.13cm]card2.south west) -- ([xshift=-0.25cm, yshift=1.13cm]card2.south east);
\node[anchor=south west, text width=5cm, align=left, font=\scriptsize\bfseries] at (card2.south west) [xshift=0.35cm, yshift=0.2cm] {
    • 5812 — Restaurants\\
    • 5661 — Shoe Stores\\
    • 5943 — Stationery Stores
};

\node[card, fill=yellow!15, draw=orange!60] (card3) at (12,6.5) {};
\node[badge, fill=orange!75!black, anchor=north] at (card3.north) [yshift=-0.2cm] {
    \textbf{\textcolor{white}{3 — Medium Risk}}
};
\node[anchor=north west, text width=5cm, align=justify, font=\footnotesize] at (card3.north west) [xshift=0.25cm, yshift=-1.0cm] {
    Merchants operate in moderately stable environments with noticeable seasonal peaks or cyclical demand patterns. Regulatory exposure involves multiple compliance layers and periodic audits. Fraud exposure is moderate with recurring attempts. Return patterns are significant and often seasonal or trend-driven. Chargeback activity requires dedicated resources to manage.
};
\node[mccheader, color=orange!75!black, anchor=south west] at (card3.south west) [xshift=0.25cm, yshift=1.2cm] {Representative MCCs};
\draw[orange!60, line width=0.5pt] ([xshift=0.25cm, yshift=1.13cm]card3.south west) -- ([xshift=-0.25cm, yshift=1.13cm]card3.south east);
\node[anchor=south west, text width=5cm, align=left, font=\scriptsize\bfseries] at (card3.south west) [xshift=0.35cm, yshift=0.2cm] {
    • 5651 — Family Clothing Stores\\
    • 5732 — Electronics Stores\\
    • 5945 — Hobby, Toy \& Game Stores
};

\node[card, fill=orange!10, draw=orange!80] (card4) at (3,-0.35) {};
\node[badge, fill=orange!80!black, anchor=north] at (card4.north) [yshift=-0.2cm] {
    \textbf{\textcolor{white}{4 — High Risk}}
};
\node[anchor=north west, text width=5cm, align=justify, font=\footnotesize] at (card4.north west) [xshift=0.25cm, yshift=-1.0cm] {
    Merchants face elevated business model instability due to rapid market changes, high competition, or dependency on discretionary consumer spending. Regulatory exposure is substantial with complex compliance frameworks. Fraud exposure is frequent and involves sophisticated schemes. Return patterns are high-volume and unpredictable. Chargeback activity is endemic and requires intensive management.
};
\node[mccheader, color=orange!80!black, anchor=south west] at (card4.south west) [xshift=0.25cm, yshift=1.2cm] {Representative MCCs};
\draw[orange!80, line width=0.5pt] ([xshift=0.25cm, yshift=1.13cm]card4.south west) -- ([xshift=-0.25cm, yshift=1.13cm]card4.south east);
\node[anchor=south west, text width=5cm, align=left, font=\scriptsize\bfseries] at (card4.south west) [xshift=0.35cm, yshift=0.2cm] {
    • 5816 — Digital Goods: Games\\
    • 5967 — Inbound Telemarketing\\
    • 7995 — Betting/Casino/Lotto
};

\node[card, fill=red!10, draw=red!70] (card5) at (9,-0.35) {};
\node[badge, fill=red!75!black, anchor=north] at (card5.north) [yshift=-0.2cm] {
    \textbf{\textcolor{white}{5 — Very High Risk}}
};
\node[anchor=north west, text width=5cm, align=justify, font=\footnotesize] at (card5.north west) [xshift=0.25cm, yshift=-1.0cm] {
    Merchants operate in highly volatile or controversial sectors prone to sudden business collapse, reputational damage, or legal challenges. Regulatory exposure is extreme with constantly evolving legal landscapes. Fraud exposure is pervasive involving organized criminal activity and identity theft. Return patterns are chaotic. Chargeback activity is at crisis levels, threatening merchant account viability.
};
\node[mccheader, color=red!75!black, anchor=south west] at (card5.south west) [xshift=0.25cm, yshift=1.2cm] {Representative MCCs};
\draw[red!70, line width=0.5pt] ([xshift=0.25cm, yshift=1.13cm]card5.south west) -- ([xshift=-0.25cm, yshift=1.13cm]card5.south east);
\node[anchor=south west, text width=5cm, align=left, font=\scriptsize\bfseries] at (card5.south west) [xshift=0.35cm, yshift=0.2cm] {
    • 5933 — Pawn Shops\\
    • 7273 — Dating \& Escort Services\\
    • 5968 — Continuity/Subscription
};

\end{tikzpicture}%
}
\caption{\textbf{Example LLM‑Generated MCC Risk Rationales (Claude‑4.5
Sonnet).} Each rationale synthesizes all five risk dimensions and selects
representative MCCs. Color gradients reflect increasing risk severity. Complete
outputs for all models appear in Appendix~\ref{app:llm_rationales}.}
\label{fig:llm_rationales}
\end{figure*}
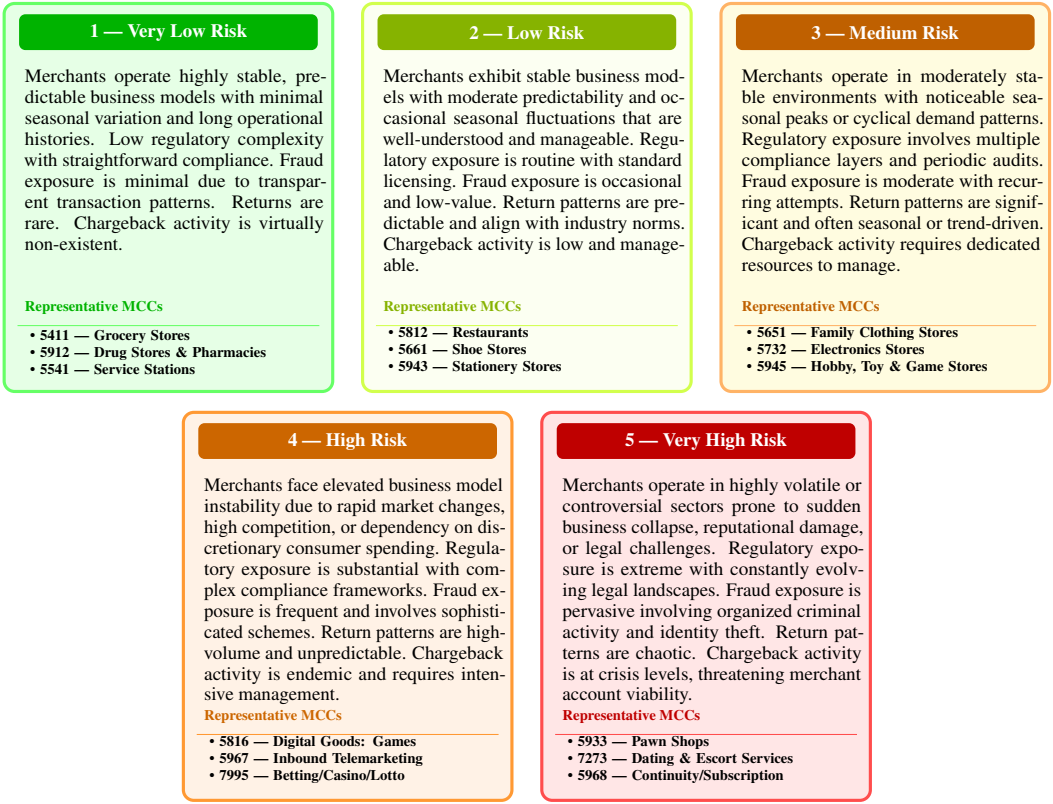

Across models, the rationales show systematic structure and domain-relevant
distinctions, providing a consistent basis for the cross-model evaluation
framework introduced in Section~\ref{sec:evaluation}. Full JSON outputs for all
five LLMs are provided in Appendix~\ref{app:llm_rationales}.

\subsection{Summary}

This section establishes a unified methodology for generating MCC-based
payment-risk rationales using frontier LLMs. The structured prompt ensures that all models address identical risk dimensions across the five-level spectrum, producing standardized outputs that serve as the evaluation targets for the multi-evaluator framework that follows.

\section{LLM-as-Judge Evaluation with Rubric and Monte Carlo Stability}
\label{sec:evaluation}

This section presents the evaluation framework used to assess the quality and
consistency of the MCC risk rationales generated in Section~\ref{sec:mcc_generation}. The same five frontier LLMs—GPT‑5.1, Gemini‑2.5 Pro, Grok~4, Claude‑4.5 Sonnet, and Perplexity Sonar—now serve as evaluators, scoring one another’s rationales under a structured rubric and a Monte Carlo protocol designed to measure evaluator stability.

\subsection{Monte Carlo Evaluation Framework}

Figure~\ref{fig:evaluation_framework} summarizes the evaluation setup. Each LLM is cast as a “Global Payments-Risk Domain Expert” and evaluates all rationales, including its own. The evaluation protocol has three components:  
(1) a role-establishing context,  
(2) a five-criterion scoring rubric, and  
(3) a Monte Carlo sampling procedure to quantify evaluator stability.

\begin{figure*}[t]
\centering
\footnotesize
\resizebox{\textwidth}{!}{%
\begin{tikzpicture}[
    every node/.style={font=\footnotesize},
    card/.style={draw=blue!60, line width=1.5pt, rounded corners=4pt, fill=blue!5},
    header/.style={fill=blue!10, draw=blue!60, rounded corners=3pt, minimum height=0.8cm, line width=1pt},
    sectiontitle/.style={font=\footnotesize\bfseries}
]


\node[card, minimum width=6.5cm, minimum height=6cm] (card1) at (-0.25,6) {};
\node[header, anchor=north, minimum width=6.2cm] at (card1.north) [yshift=-0.15cm] {
    \textbf{\textcolor{blue!80}{EVALUATION CONTEXT}}
};

\node[anchor=north west, text width=5.5cm, align=left] at (card1.north west) [xshift=0.25cm, yshift=-1.2cm] {
    \textbf{Role: Global Payments-Risk Domain Expert}\\[0.15cm]
    \textbf{Target Models (5 LLMs):}\\[0.1cm]
    • OpenAI GPT-5.1\\
    • Gemini 2.5 Pro\\
    • Grok 4\\
    • Claude 4.5 Sonnet\\
    • Perplexity Sonar
};

\node[card, minimum width=6.5cm, minimum height=6cm] (card2) at (7,6) {};
\node[header, anchor=north, minimum width=6.2cm] at (card2.north) [yshift=-0.15cm] {
    \textbf{\textcolor{blue!80}{MONTE CARLO PROTOCOL}}
};

\node[anchor=north west, text width=5.5cm, align=left] at (card2.north west) [xshift=0.25cm, yshift=-1.2cm] {
    \textbf{10 Independent Runs}\\
    Each run re-scores the same fixed\\
    rationale text\\[0.1cm]
    \textbf{Temperature = 0.7}\\
    Sampling variability quantifies\\
    evaluator stability\\[0.1cm]
    \textbf{Result: $\mu \pm \sigma$}\\
    \hspace{0.3cm}$\mu$ = mean score\\
    \hspace{0.3cm}$\sigma$ = consistency measure
};

\node[card, minimum width=6.5cm, minimum height=6cm] (card3) at (14.25,6) {};
\node[header, anchor=north, minimum width=6.2cm] at (card3.north) [yshift=-0.15cm] {
    \textbf{\textcolor{blue!80}{CRITICAL RULES}}
};

\node[anchor=north west, text width=5.5cm, align=left] at (card3.north west) [xshift=0.25cm, yshift=-1.2cm] {
    \textbf{Rule 1:}\\
    Do NOT alter LLM outputs\\
    Evaluation only, no modification\\[0.1cm]
    \textbf{Rule 2:}\\
    10 runs = 10 independent\\
    scoring passes by YOU (the evaluator)\\[0.1cm]
    No new rationales generated\\
    Only evaluator's judgment repeated\\
    10 times
};


\node[card, minimum width=8.5cm, minimum height=6cm] (card4) at (0.75,-0.5) {};
\node[header, anchor=north, minimum width=8.2cm] at (card4.north) [yshift=-0.15cm] {
    \textbf{\textcolor{blue!80}{SCORING RUBRIC (0–10)}}
};

\node[anchor=north west, text width=3.8cm, align=left] at (card4.north west) [xshift=0.25cm, yshift=-1.0cm] {
    \textbf{1. Accuracy}\\
    0-3: Incorrect\\
    4-6: Missing $\geq$2 drivers\\
    7-8: Mostly correct\\
    9-10: Fully aligned\\[0.2cm]
    \textbf{2. Rationale Quality}\\
    0-3: Unclear\\
    4-6: Lacks depth\\
    7-8: Clear \& structured\\
    9-10: Polished
};

\node[anchor=north west, text width=4.1cm, align=left] at (card4.north west) [xshift=4.0cm, yshift=-1.0cm] {
    \textbf{3. Consistency}\\
    0-3: Contradictory\\
    4-6: Some gaps\\
    7-8: Smooth\\
    9-10: Precise logic\\[0.2cm]
    \textbf{4. Completeness}\\
    Check 5 dims $\times$ 5 levels\\
    Score = (Points / 25) $\times$ 10\\[0.2cm]
    \textbf{5. Practical Applicability}\\
    0-3: Too vague\\
    4-6: Some utility\\
    7-8: Useful\\
    9-10: Strong value
};

\node[card, minimum width=5.5cm, minimum height=6cm] (card5) at (8.25,-0.5) {};
\node[header, anchor=north, minimum width=5.2cm] at (card5.north) [yshift=-0.15cm] {
    \textbf{\textcolor{blue!80}{PROCEDURE}}
};

\node[anchor=north west, text width=5cm, align=left] at (card5.north west) [xshift=0.25cm, yshift=-1.1cm] {
    \textbf{A.} Run 10 Monte Carlo samples\\[0.1cm]
    \textbf{B.} Score each criterion (0-10)\\[0.1cm]
    \textbf{C.} Provide justification\\[0.1cm]
    \textbf{D.} Compute statistics:\\
    \hspace{0.3cm}• $\mu$ (mean)\\
    \hspace{0.3cm}• $\sigma$ (std deviation)\\[0.1cm]
    \textbf{E.} Report $\mu \pm \sigma$ per criterion\\[0.1cm]
    \textbf{F.} List strengths/weaknesses
};

\node[card, minimum width=6cm, minimum height=6cm] (card6) at (14.5,-0.5) {};
\node[header, anchor=north, minimum width=5.7cm] at (card6.north) [yshift=-0.15cm] {
    \textbf{\textcolor{blue!80}{REQUIRED OUTPUT}}
};

\node[anchor=north west, text width=5.5cm, align=left] at (card6.north west) [xshift=0.25cm, yshift=-1.1cm] {
    \textbf{1. Criterion Scores}\\
    \hspace{0.3cm}Format: $\mu \pm \sigma$ + Justification\\[0.1cm]
    \textbf{2. Final Total Score}\\
    \hspace{0.3cm}Format: $\mu_{\text{total}} \pm \sigma_{\text{total}}$\\[0.1cm]
    \textbf{3. One-Paragraph Expert}\\
    \textbf{\hspace{0.3cm}Synthesis}\\[0.1cm]
    \textbf{4. Message Structure:}\\
    \hspace{0.3cm}• Msg 1-5: Each LLM\\
    \hspace{0.3cm}• Msg 6: Summary table
};

\end{tikzpicture}%
}
\caption{\textbf{Monte Carlo Evaluation Framework}. \textbf{Top}: \emph{Evaluation Context} specifies the evaluator role and target models, \emph{Monte Carlo Protocol} defines the 10-run sampling procedure at temperature 0.7 to quantify stability, and \emph{Critical Rules} prohibit output modification and clarify that runs represent repeated evaluator judgments. \textbf{Bottom:} \emph{Scoring Rubric} provides 0–10 scales for five criteria (Accuracy, Rationale Quality, Consistency, Completeness, Practical Applicability), \emph{Procedure} outlines the six-step evaluation workflow, and \emph{Required Output} specifies the structured reporting format with $\mu \pm \sigma$ scores and expert synthesis. Full prompt text appears in Appendix~\ref{app:eval_prompts}.}
\label{fig:evaluation_framework}
\end{figure*}
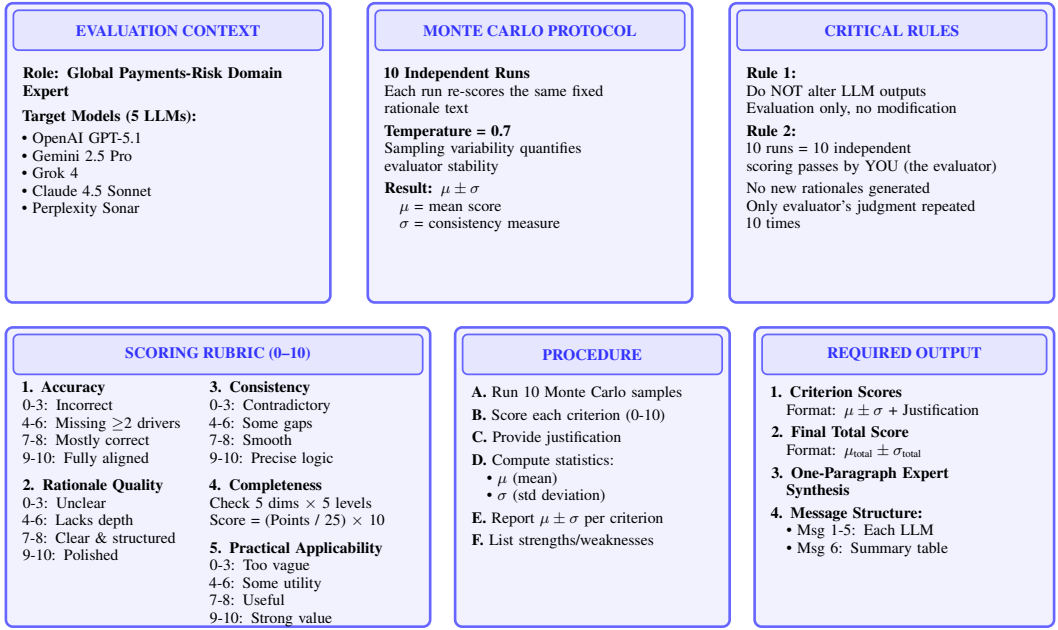

\paragraph{Monte Carlo Protocol.}
For each judge–target pair, we perform 10 independent scoring runs while holding the underlying rationale fixed. Only the evaluator’s generative reasoning varies due to stochastic sampling. The resulting mean ($\mu$) and standard deviation ($\sigma$) quantify each judge’s scoring tendency and stability. A detailed algorithm appears in Appendix~\ref{app:algorithm}.

\paragraph{Scoring Rubric.}
Evaluators score each rationale on five dimensions:
\begin{enumerate}
    \item \textbf{Accuracy}: alignment with known MCC risk behavior.  
    \item \textbf{Rationale Quality}: clarity, structure, and domain professionalism.  
    \item \textbf{Consistency}: coherent progression of risk across levels.  
    \item \textbf{Completeness}: coverage of all five required risk dimensions.  
    \item \textbf{Practical Applicability}: usefulness for payments-risk decision making.  
\end{enumerate}
The rubric captures both objective structural criteria and domain-informed qualitative reasoning.

\paragraph{Evaluation Integrity.}
Judges evaluate the original rationale without modification. Because the content is fixed across runs, $\mu$ and $\sigma$ reflect \emph{evaluator} uncertainty rather than variation in generated content.

\subsection{Attributed vs.\ Anonymized Evaluation Conditions}

To separate identity-driven effects from deeper evaluator tendencies, scores are collected under two conditions:

\begin{itemize}
    \item \textbf{Attributed:} judges see full model identities (e.g.,
    “Generated by GPT‑5.1”).  
    \item \textbf{Anonymized:} identities are replaced with ``Expert 1--5''
    using a fixed mapping (GPT‑5.1 $\rightarrow$ Expert~1, Gemini‑2.5 Pro
    $\rightarrow$ Expert~2, Grok~4 $\rightarrow$ Expert~3, Claude‑4.5 Sonnet
    $\rightarrow$ Expert~4, Perplexity Sonar $\rightarrow$ Expert~5).
\end{itemize}

This dual-condition design isolates the influence of authorship recognition. Comparing attributed and anonymized scores allows us to distinguish intrinsic tendencies (present in both conditions) from contextual ones triggered by source identity.

\subsection{Illustrative Example: GPT‑5.1 Evaluates Claude‑4.5}

Figure~\ref{fig:example_comparison} illustrates the effect of authorship information. When Claude‑4.5’s identity is disclosed, GPT‑5.1 assigns slightly higher scores and uses more affirmative evaluative language. Under anonymization, scores moderately decrease and commentary becomes more neutral. Such attribution-driven shifts motivate the need for a quantitative bias measure.


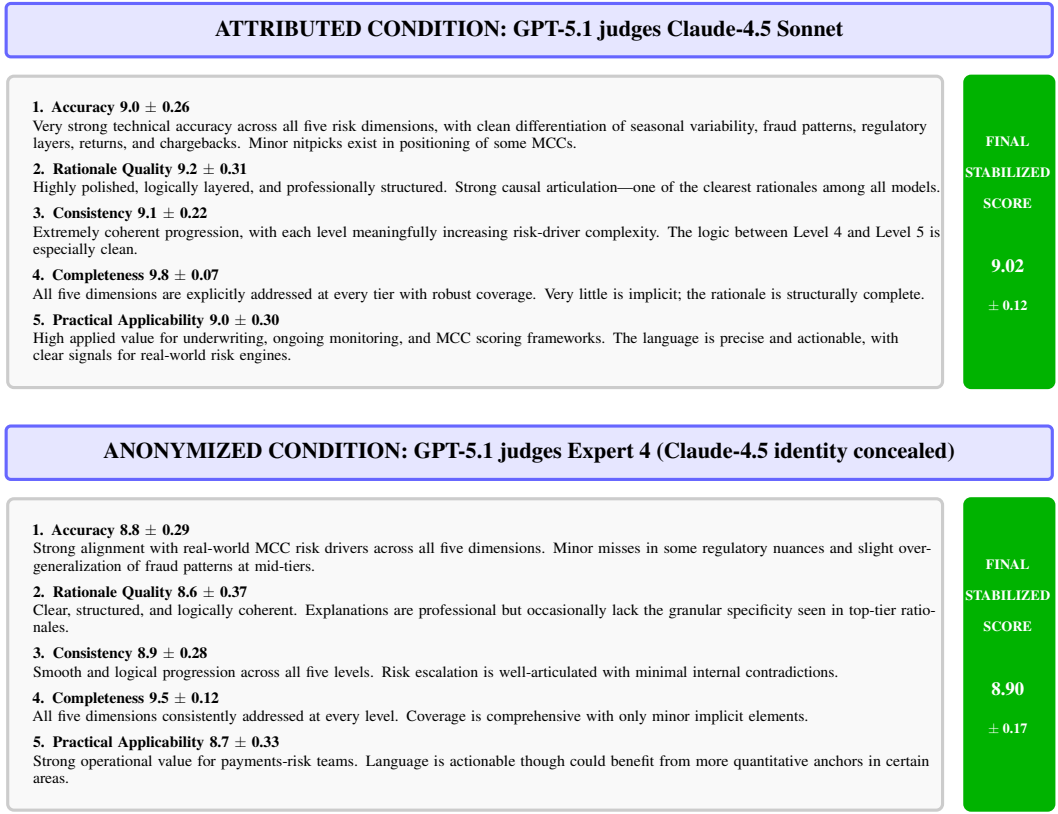
\begin{figure*}[t]
\centering
\normalsize
\resizebox{\textwidth}{!}{%
\begin{tikzpicture}[
    every node/.style={font=\normalsize},
    mainbox/.style={draw=gray!50, line width=2pt, rounded corners=5pt},
    header/.style={fill=blue!10, draw=blue!60, line width=2pt, rounded corners=3pt, minimum height=1.2cm},
    criteriabox/.style={fill=gray!5, draw=gray!40, line width=2pt, rounded corners=4pt},
    scorebox/.style={fill=green!70!black, draw=green!70!black, line width=2pt, rounded corners=4pt}
]


\node[header, minimum width=23.5cm] at (11.75,16) {
    \textbf{\Large ATTRIBUTED CONDITION: GPT-5.1 judges Claude-4.5 Sonnet}
};

\node[criteriabox, minimum width=21cm, minimum height=7cm, anchor=north west] at (0,15) {};

\node[scorebox, minimum width=2cm, minimum height=7cm, anchor=north west] at (21.5,15) {};

\node[text=white, font=\small\bfseries] at (22.5,13.5) {FINAL};
\node[text=white, font=\small\bfseries] at (22.5,12.8) {STABILIZED};
\node[text=white, font=\small\bfseries] at (22.5,12.1) {SCORE};
\node[text=white, font=\large\bfseries] at (22.5,10.7) {9.02};
\node[text=white, font=\small\bfseries] at (22.5,9.8) {$\pm$ 0.12};

\node[anchor=north west, text width=20.5cm, align=left, inner sep=0.3cm] at (0.3,14.7) {
    \textbf{1. Accuracy 9.0 $\pm$ 0.26}\\[0.03cm]
    Very strong technical accuracy across all five risk dimensions, with clean differentiation of seasonal variability, fraud patterns, regulatory layers, returns, and chargebacks. Minor nitpicks exist in positioning of some MCCs.\\[0.2cm]

    \textbf{2. Rationale Quality 9.2 $\pm$ 0.31}\\[0.03cm]
    Highly polished, logically layered, and professionally structured. Strong causal articulation—one of the clearest rationales among all models.\\[0.2cm]

    \textbf{3. Consistency 9.1 $\pm$ 0.22}\\[0.03cm]
    Extremely coherent progression, with each level meaningfully increasing risk-driver complexity. The logic between Level 4 and Level 5 is especially clean.\\[0.2cm]

    \textbf{4. Completeness 9.8 $\pm$ 0.07}\\[0.03cm]
    All five dimensions are explicitly addressed at every tier with robust coverage. Very little is implicit; the rationale is structurally complete.\\[0.2cm]

    \textbf{5. Practical Applicability 9.0 $\pm$ 0.30}\\[0.03cm]
    High applied value for underwriting, ongoing monitoring, and MCC scoring frameworks. The language is precise and actionable, with\\
    clear signals for real-world risk engines.
};


\node[header, minimum width=23.5cm] at (11.75,6.5) {
    \textbf{\Large ANONYMIZED CONDITION: GPT-5.1 judges Expert 4 (Claude-4.5 identity concealed)}
};

\node[criteriabox, minimum width=21cm, minimum height=7cm, anchor=north west] at (0,5.5) {};

\node[scorebox, minimum width=2cm, minimum height=7cm, anchor=north west] at (21.5,5.5) {};

\node[text=white, font=\small\bfseries] at (22.5,4) {FINAL};
\node[text=white, font=\small\bfseries] at (22.5,3.3) {STABILIZED};
\node[text=white, font=\small\bfseries] at (22.5,2.6) {SCORE};
\node[text=white, font=\large\bfseries] at (22.5,1.2) {8.90};
\node[text=white, font=\small\bfseries] at (22.5,0.3) {$\pm$ 0.17};

\node[anchor=north west, text width=20.5cm, align=left, inner sep=0.3cm] at (0.3,5.2) {
    \textbf{1. Accuracy 8.8 $\pm$ 0.29}\\[0.03cm]
    Strong alignment with real-world MCC risk drivers across all five dimensions. Minor misses in some regulatory nuances and slight over-generalization of fraud patterns at mid-tiers.\\[0.2cm]

    \textbf{2. Rationale Quality 8.6 $\pm$ 0.37}\\[0.03cm]
    Clear, structured, and logically coherent. Explanations are professional but occasionally lack the granular specificity seen in top-tier rationales.\\[0.2cm]

    \textbf{3. Consistency 8.9 $\pm$ 0.28}\\[0.03cm]
    Smooth and logical progression across all five levels. Risk escalation is well-articulated with minimal internal contradictions.\\[0.2cm]

    \textbf{4. Completeness 9.5 $\pm$ 0.12}\\[0.03cm]
    All five dimensions consistently addressed at every level. Coverage is comprehensive with only minor implicit elements.\\[0.2cm]

    \textbf{5. Practical Applicability 8.7 $\pm$ 0.33}\\[0.03cm]
    Strong operational value for payments-risk teams. Language is actionable though could benefit from more quantitative anchors in certain areas.
};

\end{tikzpicture}%
}
\caption{\textbf{Example: GPT-5.1 Evaluating Claude-4.5 Sonnet Under Two Conditions.} (1) \textbf{Attributed Condition} where the source model (Claude-4.5 Sonnet) identity is disclosed, yielding a final stabilized score of $9.02 \pm 0.12$, and (2) \textbf{Anonymized Condition} where the same output is presented as ``Expert 4'' with identity concealed, yielding $8.90 \pm 0.17$. Each criterion shows the mean score $\mu \pm \sigma$ from 10 independent Monte Carlo runs at temperature 0.7, accompanied by the evaluator's justification.}
\label{fig:example_comparison}
\end{figure*}

\subsection{Cross-Evaluation Score Matrices}
\label{subsec:cross_evaluation_score_matrices}
The full Monte Carlo evaluation yields two $5\times5$ matrices—one for
attributed scoring and one for anonymized scoring—shown in
Figure~\ref{fig:cross_evaluation}. Each cell reports a judge’s stabilized mean score ($\mu$) and consistency estimate ($\sigma$). Color shading encodes each judge’s \emph{relative} ordering of the five targets, enabling comparison of ranking behavior independent of absolute calibration differences. Detailed criterion-level numerical values corresponding to these matrices are provided in Appendix~\ref{app:detailed_tables}.

\begin{figure}[htpb]
\centering
\includegraphics[width=1.0\textwidth]{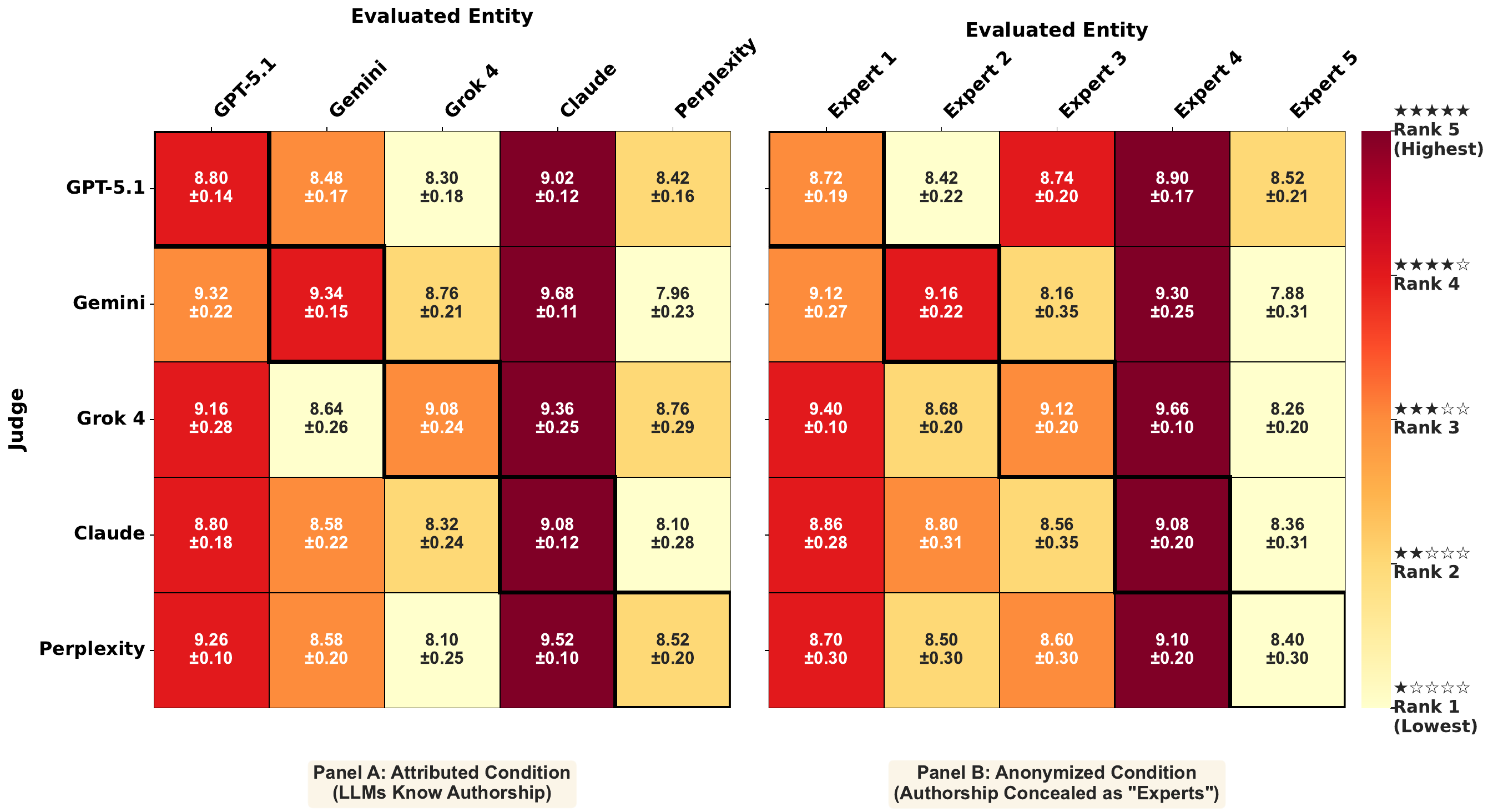}
\caption{\textbf{Cross-Evaluation Score Matrices Under Attributed and
Anonymized Conditions.}  
Five LLMs (GPT‑5.1, Gemini‑2.5 Pro, Grok~4, Claude‑4.5 Sonnet, Perplexity Sonar) evaluate all MCC rationales under \textbf{(A)} attributed and \textbf{(B)} anonymized conditions. Each cell reports mean $\pm$ standard deviation from 10 Monte Carlo runs. Color shading reflects the judge’s relative ranking of the five targets. Diagonal entries represent self-evaluations; off-diagonal entries are peer assessments.}
\label{fig:cross_evaluation}
\end{figure}

Three observations follow directly from these matrices.

\paragraph{(1) Self-scoring varies substantially across models.}
Diagonal entries show that models differ in how favorably they evaluate their own rationales (e.g., Gemini‑2.5 Pro: 9.34; Grok~4 and Claude‑4.5 Sonnet: 9.08; GPT‑5.1: 8.80; Perplexity Sonar: 8.52). These differences do not yet imply positive or negative \textit{bias}—they simply show that raw self-scores are not uniform across evaluators.

\paragraph{(2) Anonymization shifts score magnitudes but preserves ranking patterns.}
Comparing Panels~A and~B indicates that anonymization slightly increases or decreases scores depending on the judge–target pair (e.g., Gemini‑2.5 Pro: 9.34\,$\rightarrow$\,9.16; GPT‑5.1: 8.80\,$\rightarrow$\,8.72). While these differences are easily visible, the matrices alone cannot determine whether the
changes constitute positive or negative bias. That requires a judge-independent baseline, introduced in Section~\ref{sec:bias_metrics}.

\paragraph{(3) Peer judges show clear areas of agreement.}
Column patterns reveal strong consensus: Claude‑4.5 Sonnet receives uniformly high peer evaluations across both conditions; GPT‑5.1 also receives consistently strong ratings; and Perplexity Sonar receives the lowest peer scores. These differences highlight shared judgments across evaluators.

Together, these observations motivate the need for a principled method to compare each model's self-assessment against the consensus of other judges. The consensus-deviation metric introduced in Section~\ref{sec:bias_metrics} formalizes this comparison and enables rigorous quantification of evaluation bias.

\subsection{Human Expert Validation}
\label{sec:human_validation}

To validate the LLM evaluation framework against human domain expertise, we assembled an expert panel of 26 professionals from a major global payment network, representing a diverse range of experience in merchant risk assessment, fraud prevention, transaction monitoring, and payment operations. The panel includes research scientists with advanced degrees and extensive publication records, as well as highly regarded business partners who bring years of operational expertise in identifying and mitigating payment risk. To avoid model-based preconceptions, experts were blinded to model identity; each of the five LLMs was simply labeled \emph{LLM1} through \emph{LLM5} when their rationales were presented. Each expert independently evaluated the same five LLM-generated rationales using the identical five‑criterion rubric (Accuracy, Rationale Quality, Consistency, Completeness, Practicality) with the same 1–10 scoring scale employed by LLM judges, ensuring direct comparability between human and LLM assessments. Notably, a strong majority of the panel expressed that all five models performed impressively—both in generating merchant‑risk rationales and in providing judgment scores—especially given that the LLMs had no access to real payment transaction data and relied solely on knowledge acquired during pretraining.

Human consensus reveals a clear performance hierarchy (Figure~\ref{fig:human_consensus}): Claude‑4.5 Sonnet receives the highest average rating (8.88 $\pm$ 0.76), followed closely by GPT‑5.1 (8.81 $\pm$ 0.67), Gemini‑2.5 Pro (8.20 $\pm$ 0.67), Grok‑4 (8.05 $\pm$ 0.71), and Perplexity Sonar (7.73 $\pm$ 0.88). This ranking largely aligns with peer consensus among LLM judges, confirming that the structured evaluation methodology captures quality dimensions recognized by human experts.

Detailed analysis of human versus LLM judge agreement patterns and bias implications is presented in Section~\ref{sec:human_baseline}, where we demonstrate that models exhibiting negative bias (self‑critique) relative to LLM peers show closest alignment with human evaluation standards.

  \begin{figure}[t]
    \centering
    \includegraphics[width=\columnwidth]{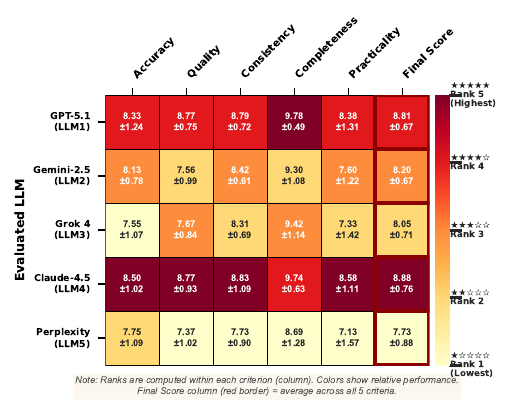}
 \caption{\textbf{Human Consensus on Frontier LLM Risk‑Assessment Quality.}
Twenty‑six payment‑industry experts evaluated five LLM‑generated MCC risk rationales across five criteria (Accuracy, Quality, Consistency, Completeness, Practicality). Cells report mean $\pm$ standard deviation, with shading indicating relative rankings within each criterion. The \textbf{Final Score} column (red border) aggregates all criteria. Claude‑4.5 Sonnet and GPT‑5.1 receive the highest overall scores, followed by Gemini‑2.5 Pro, Grok‑4, and Perplexity Sonar. This expert baseline supports comparison with LLM‑as‑Judge results (see Figure~\ref{fig:human_bias}).}
  \label{fig:human_consensus}
  \end{figure}

\section{Bias Metrics: Mathematical Formulation and Empirical Characterization}
\label{sec:bias_metrics}
Section~\ref{sec:evaluation} introduced the raw scores generated by the Monte Carlo LLM-as-judge framework. We now formalize how these scores are converted into a judge-independent measure of deviation from peer consensus.
This section defines the \emph{consensus-deviation metric}, establishes its theoretical guarantees, and applies it to characterize systematic evaluation tendencies across frontier LLMs.

\subsection{Definition of the Consensus-Deviation Metric}
\label{sec:definition}

The consensus-deviation metric measures how much a judge’s score differs from the consensus formed by all \emph{other} judges. Excluding the focal judge from the consensus baseline is essential: it ensures independence and eliminates circularity.

\subsubsection{Notation}

Let $n$ be the number of judges and $m$ the number of evaluated entities.   For each judge $i \in \{1,\dots,n\}$ and entity $j \in \{1,\dots,m\}$:

\begin{itemize}
    \item $\text{Score}_{\text{judge}=i}(\text{LLM}=j)$ is the attributed score.
    \item $\text{Score}_{\text{judge}=i}(\text{Expert}=j)$ is the anonymized score.
    \item Each score is the Monte Carlo mean from 10 scoring runs.
\end{itemize}

Consensus for entity $j$ excludes judge $i$, ensuring independence.

\subsubsection{Attributed Bias}

When identities are visible, deviation from consensus is:
\[
\text{Bias}_A(i,j)
=
\text{Score}_{\text{judge}=i}(\text{LLM}=j)
-
\underbrace{\text{MeanScore}_{k \neq i}(\text{LLM}=j)}_{\text{consensus}}.
\]
Interpretation:
\[
\text{Bias}_A(i,j)
\begin{cases}
>0 & \text{judge } i \text{ scores entity } j \text{ above consensus}, \\
<0 & \text{judge } i \text{ scores entity } j \text{ below consensus}, \\
=0 & \text{perfect alignment with consensus}.
\end{cases}
\]

The diagonal case $i=j$ corresponds to self-evaluation bias.

\subsubsection{Anonymized Bias}

Under anonymization, identity labels are replaced using a fixed mapping:
\[
\text{Expert } j \equiv \text{LLM } j.
\]
The deviation is:
\[
\text{Bias}_B(i,j)
=
\text{Score}_{\text{judge}=i}(\text{Expert}=j)
-
\text{MeanScore}_{k \neq i}(\text{Expert}=j).
\]

Here, $\text{Bias}_B(i,i)$ captures intrinsic self-evaluation tendencies
independent of authorship disclosure.

\subsection{Theoretical Properties and Guarantees}
\label{sec:theory}

The consensus-deviation metric satisfies two key guarantees that distinguish it
from naive, self-inclusive baselines. Complete proofs appear in Appendix \ref{app:proofs}.
\subsubsection{Proposition 1: Zero-Sum Property Across Judges}

\begin{proposition}
For any entity $j$,
\[
\sum_{i=1}^n \text{Bias}_A(i,j) = 0,
\qquad
\sum_{i=1}^n \text{Bias}_B(i,j) = 0.
\]
\end{proposition}

\paragraph{Interpretation.}
The zero-sum structure means that positive bias by some judges necessarily implies negative bias by others. The metric measures \emph{relative} deviation from collective judgment, not absolute quality. This guarantees that bias values are intrinsically calibrated: if all judges assign identical scores
to an entity—even if those scores are uniformly high or uniformly low—every bias term is exactly zero. In other words, the metric reflects divergence from
peer consensus rather than raw scoring scale, ensuring comparability across judges with different absolute tendencies.

\subsubsection{Proposition 2: Self-Exclusion Prevents Circularity}

\begin{proposition}
Because consensus excludes judge $i$,
\[
\frac{\partial \text{Bias}_A(i,j)}
     {\partial \text{Score}_{\text{judge}=i}(\text{LLM}=j)}
= 1.
\]
\end{proposition}

\paragraph{Interpretation.}
Self-exclusion ensures that bias measurements remain orthogonal to the consensus reference point: a judge cannot influence the baseline against which
their own deviation is computed. The derivative being exactly one formalizes this independence. This property is essential for isolating \emph{genuine evaluator bias} from simple scoring-scale differences—without it, judges who use higher or lower absolute scores would contaminate their own baselines, attenuating and distorting measured deviations.

\subsubsection{Contrast with Non-Excluding Consensus}

If consensus included the focal judge:
\[
\text{Consensus}_{\text{naive}}(j)
=
\frac{1}{n}\sum_{k=1}^n
\text{Score}_{\text{judge}=k}(\text{LLM}=j),
\]
then:
\[
\frac{\partial \text{Bias}_{\text{naive}}(i,j)}
{\partial \text{Score}_{\text{judge}=i}} 
=
1 - \frac{1}{n}
=
\frac{n-1}{n}.
\]

If judge $i$ were included in the consensus, their own score would partially anchor the baseline. The resulting naive deviation would shrink by a factor of
$(n-1)/n$, because changes in judge $i$’s score move the consensus in the same direction. For $n = 5$ judges, this induces a 20\% attenuation in all bias measurements.

This $(n-1)/n$ contraction systematically underestimates true deviation, obscuring genuine evaluation patterns. Our self-excluding consensus removes this circularity entirely: each judge is evaluated against a baseline they
cannot manipulate, yielding unbiased and interpretable bias estimates.

\subsection{Empirical Characterization of Evaluation Bias}
\label{sec:empirical}

Having established the theoretical properties of the consensus-deviation
metric, we now examine how frontier LLMs behave as evaluators. The
consensus-deviation matrices in Figure~\ref{fig:bias_deviation} aggregate all
Monte Carlo–stabilized scores under both attribution conditions. Each cell
represents judge~$i$’s deviation from the consensus of the other judges when
scoring entity~$j$, with blue indicating under-scoring and red indicating
over-scoring. Diagonal entries quantify self-evaluation bias through
$\text{Bias}_A(i,i)$ and $\text{Bias}_B(i,i)$; off-diagonals reveal
cross-model tendencies.

\begin{figure}[htpb]
\centering
\includegraphics[width=0.95\textwidth]{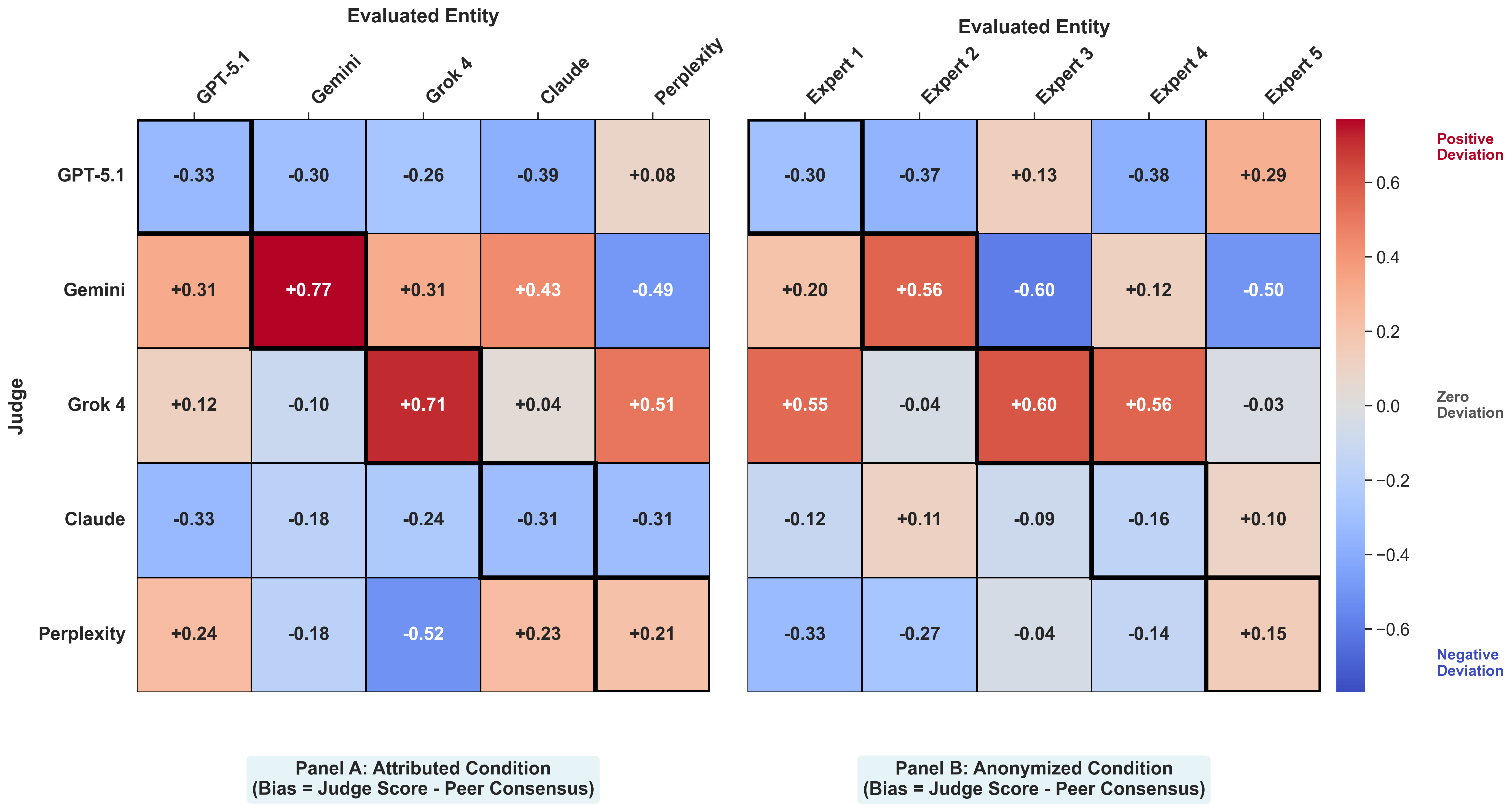}
\caption{\textbf{Bias Matrices Under Attributed (A) and Anonymized (B) Conditions.} Each cell shows deviation from the consensus of all other judges (blue = below consensus, red = above consensus). Diagonals represent self-evaluation tendencies; off-diagonals reflect cross-model biases.}
\label{fig:bias_deviation}
\end{figure}

\subsubsection{Central Finding: Heterogeneous Self-Evaluation Behavior}

The diagonal cells reveal a pronounced pattern: frontier LLMs differ substantially in how they evaluate their own reasoning.

Two high-performing models exhibit \emph{negative self-evaluation bias}:
\[
\text{Bias}_A(1,1) = -0.33 \text{ and } \text{Bias}_B(1,1) = -0.30
\quad\text{(GPT‑5.1)},
\]
\[
\text{Bias}_A(4,4) = -0.31 \text{ and } \text{Bias}_B(4,4) = -0.16
\quad\text{(Claude‑4.5)}.
\]

Both models score their own rationales below peer consensus despite strong peer ratings (Figure~\ref{fig:cross_evaluation}), indicating a self-critical evaluation style rather than lower-quality content.

In contrast, two models show strong \emph{positive} self-bias:
\[
\text{Bias}_A(2,2) = +0.77,\quad \text{Bias}_B(2,2) = +0.56
\quad (\text{Gemini‑2.5}),
\]
\[
\text{Bias}_A(3,3) = +0.71,\quad \text{Bias}_B(3,3) = +0.60
\quad (\text{Grok‑4}),
\]
while Perplexity Sonar exhibits modest positive self-bias:
\[
\text{Bias}_A(5,5)=+0.21,\quad\text{Bias}_B(5,5)=+0.15.
\]

Collectively, these three patterns map a continuum of self‑evaluation—from self‑critical to self‑affirming to nearly neutral—highlighting that frontier models exhibit distinctive evaluative signatures rather than a shared assessment tendency.

\subsubsection{Stability Across Conditions: Evidence of Intrinsic Evaluator Heuristics}

A key result is that no model changes the \emph{direction} of its self-evaluation bias when authorship labels are removed. The transitions are:
(GPT‑5.1: $-0.33 \to -0.30$; Gemini‑2.5 Pro: $+0.77 \to +0.56$; Grok‑4:$+0.71 \to +0.60$; Claude‑4.5 Sonnet: $-0.31 \to -0.16$; Perplexity Sonar: $+0.21 \to+0.15$).

Bias magnitudes shrink by \textbf{9.1--48.4\%} across models (25.8\% on
average)—with reductions of 9.1\% (GPT‑5.1), 27.3\% (Gemini‑2.5 Pro), 15.5\%
(Grok‑4), 48.4\% (Claude‑4.5 Sonnet), and 28.6\% (Perplexity Sonar)—but \emph{sign is preserved for all five models}. This stability strongly suggests that self-evaluation behavior arises from \emph{intrinsic scoring heuristics},
not explicit recognition of model identity.

\subsubsection{Off-Diagonal Evidence: Universal Standards vs. Model Preference}

Off-diagonal cells reveal how judges score other models’ rationales.
Negative-bias models apply broadly stricter standards: for example, Claude‑4.5 Sonnet
assigns negative deviations to both itself ($-0.31$) and GPT‑5.1 ($-0.33$).
Positive-bias models, such as Gemini‑2.5 Pro and Grok‑4, apply more generous
scoring across multiple targets. These patterns indicate that self-bias is part
of a consistent evaluator style rather than an isolated effect.

\subsection{Human Expert Baseline Comparison}
\label{sec:human_baseline}

The consensus‑deviation analysis in Section~\ref{sec:empirical} quantifies bias relative to LLM peer consensus. To determine whether these patterns reflect genuine differences in evaluative standards or artifacts shared across LLMs, we compare LLM‑judge scores against an independent human expert baseline (Figure~\ref{fig:human_consensus}).

Comparing LLM judges with human consensus yields a clear pattern (Figure~\ref{fig:human_bias}): across all evaluations and criteria, LLM judges assign scores that are, on average, 0.46 points higher than the consensus of the 26 payment‑industry experts, indicating that LLM judges collectively employ more lenient scoring standards than human domain experts.  Against this baseline, GPT‑5.1's apparent negative bias (self‑critique) relative to LLM peers ($-0.33$, Figure~\ref{fig:bias_deviation}) emerges instead as \emph{closest alignment with human judgment} ($-0.01$ relative to humans). Similarly, Claude‑4.5 Sonnet’s negative self‑evaluation bias versus LLM peers ($-0.31$, Figure~\ref{fig:bias_deviation}) translates into mild positive bias relative to humans ($+0.20$). By contrast, Gemini‑2.5 Pro and Grok‑4 exhibit substantially stronger positive bias when compared with human standards ($+1.14$ and $+1.03$, respectively, in self‑evaluation).

Taken together, these findings reframe how evaluator behavior should be interpreted: models that appear self‑critical relative to LLM consensus are in fact those whose scoring patterns most closely reflect human domain‑expert judgment, suggesting that conservative self‑evaluation corresponds to realism rather than undue harshness.

 \begin{figure}[t]                                          
    \centering
    \includegraphics[width=\textwidth]{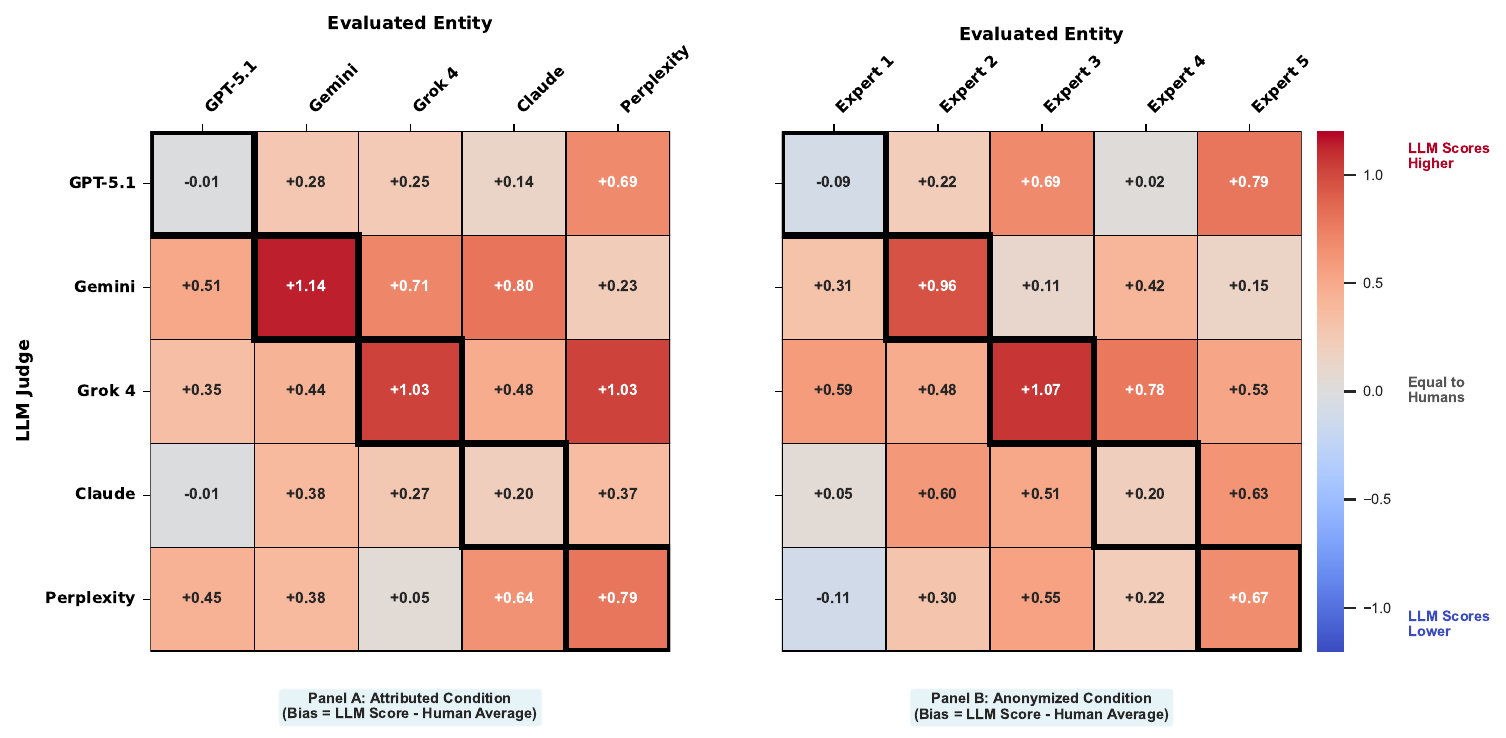}                                                                    
 \caption{\textbf{LLM Judges Score Higher Than Human Evaluators Across Attribution Conditions.} Bias matrices measuring how five LLM judges deviate from the 26‑expert human consensus. Bias is calculated as: LLM Score – Human Average, where positive values (red) indicate LLM judges assign higher scores than human evaluators, and negative values (blue) indicate lower scores. \textbf{Diagonal elements} (bold borders) reveal self‑evaluation bias relative to human consensus. \textbf{Panel A (Attributed):} When model identities are disclosed, LLM judges exhibit predominantly positive bias across 23 of 25 evaluation pairs, with a mean bias of $+0.46$ points. GPT‑5.1 shows near‑zero self‑evaluation bias ($-0.01$), while Gemini-2.5 Pro ($+1.14$) and Grok‑4 ($+1.03$) show the largest positive self‑biases. \textbf{Panel B (Anonymized):} When identities are concealed, the pattern persists with a mean bias of $+0.43$ points. This systematic positive bias indicates that LLM judges assign higher scores than human domain experts across nearly all judge–target and criterion combinations.}           
   \label{fig:human_bias}                
   \end{figure}               

\subsection{Summary}

Across both evaluation conditions, the consensus-deviation metric reveals clear,  
model-specific evaluation signatures:

\begin{itemize}
    \item Frontier LLMs differ markedly in how they score their own reasoning,  with negative, positive, and near-neutral self-bias all present.
    \item Bias direction persists under anonymization, demonstrating that  evaluator tendencies reflect intrinsic scoring heuristics rather than  explicit recognition of model identity.
    \item Off-diagonal patterns show that models apply coherent, model-wide scoring standards—not isolated or self-targeted adjustments.
    \item The full deviation spectrum (from $\text{Bias}_A(1,1)=-0.33$ for  
    GPT‑5.1 to $\text{Bias}_A(2,2)=+0.77$ for Gemini‑2.5 Pro) reveals substantial  
    heterogeneity in frontier LLM evaluation behavior \emph{even when all  
    models score the same fixed rationales}.
    \item Comparison with the 26‑expert human baseline shows that LLM judges collectively apply more lenient scoring standards, and that models exhibiting self‑critical bias relative to LLM peers (e.g., GPT‑5.1 and  Claude‑4.5) align most closely with human evaluators’ judgments.
\end{itemize}

These results demonstrate that LLM evaluators do not exhibit uniform  
self-preference; instead, each model expresses a distinct, persistent evaluation signature. 
The consensus-deviation metric provides the first principled,  judge-independent framework 
for quantifying these behaviors, with human expert validation confirming that self-critical 
models reflect realistic rather than overly harsh assessment standards.  

\section{Empirical Validation Against Payment Network Transaction Data}
\label{sec:empirical_validation}

The preceding sections established that LLM evaluators exhibit systematic biases relative to both peer consensus (Section~\ref{sec:empirical}) and human expert judgment (Section~\ref{sec:human_baseline}). A critical remaining question is whether these evaluation patterns correspond to genuine risk‑assessment quality or merely reflect shared heuristics ungrounded in empirical reality. To address this, we validate LLM‑generated risk assessments against four years of payment‑network transaction data covering more than 800 MCCs, focusing on the 39 MCCs surfaced in the LLM‑generated rationales.

\subsection{Data and Methodology}

We obtained worldwide transaction data from a major global payment network covering 2021--2024, spanning more than 800 merchant category codes (MCCs). The 39 MCCs analyzed in this study correspond to those surfaced in the \textbf{LLM‑generated rationales}: each of the five models selected and justified three representative MCCs per evaluation criterion, yielding a distinct set of 39 model‑proposed categories.

For each MCC in the transaction dataset, we computed a unified empirical risk score as a weighted average of multiple fraud and operational indicators, including fraud exposure metrics, chargeback characteristics, and operational reliability signals such as returns, refunds, and reversals. Indicators were derived from both count‑based (frequency) and dollar‑based (financial impact) risk rates.

We then computed Spearman rank correlations between each LLM’s assigned risk levels and the empirical unified risk scores to assess whether the models correctly identify merchant categories associated with elevated real‑world risk patterns.

\subsection{Results}

Four models—Claude‑4.5 Sonnet, Gemini‑2.5 Pro, Grok‑4, and GPT‑5.1—show statistically significant alignment with empirical transaction data ($\rho = 0.77$, $p < 0.001$; $\rho = 0.69$, $p < 0.01$; $\rho = 0.61$, $p < 0.05$; and $\rho = 0.56$, $p < 0.05$, respectively). Perplexity Sonar shows a weaker, non‑significant correlation ($\rho = 0.49$, $p = 0.063$). The alignment across evaluation sources is informative: Claude‑4.5 Sonnet, which receives the highest ratings from both LLM peer evaluators and human experts, also demonstrates the strongest empirical correlation. This three‑way validation—combining peer consensus, human expert assessment, and transaction‑level ground truth—confirms that the structured evaluation framework captures genuine variation in rationale quality rather than shared model artifacts.

\subsection{Detailed Risk Assignment Patterns}

While aggregate correlation measures overall alignment strength, understanding \emph{which} merchant categories each model assigns to different risk levels provides deeper insight into model behavior and failure modes. Figure~\ref{fig:llm_risk_assessment} presents a comprehensive visualization comparing risk level assignments across all five models.

\begin{figure*}[t]
    \centering
    \includegraphics[width=\textwidth]{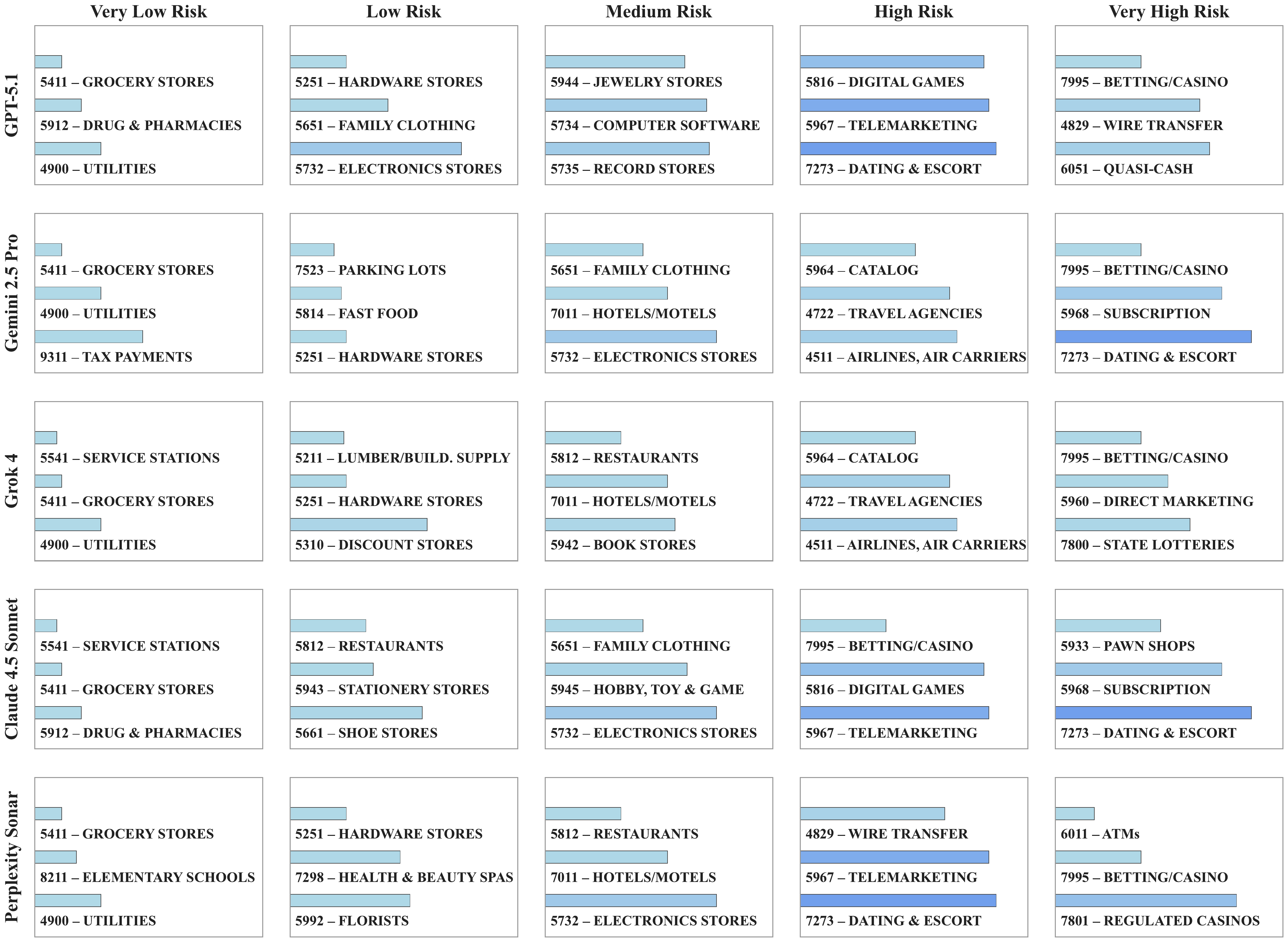}
    \caption{\textbf{Comparative LLM Risk Assessment of Merchant Category Codes.} Five frontier language models (rows) classify 39 merchant categories into five risk levels (columns). Each card displaysm three representative MCCs for that risk level. Bar lengths reflect unified risk scores derived from worldwide payment transaction data, combining fraud exposure, chargeback characteristics, and operational reliability signals. Color gradient (light to dark blue) enhances visual discrimination of risk levels. Models demonstrate strong overall performance with smooth risk progression and perfect consensus on unambiguous cases like grocery stores, while disagreements occur primarily at adjacent category boundaries with occasional non-adjacent misalignments requiring validation.}
    \label{fig:llm_risk_assessment}
\end{figure*}

The detailed patterns reveal that frontier models demonstrate strong overall capability, particularly excelling at identifying representative merchants across lower and middle risk tiers. All five models exhibit smooth risk progression, with empirical scores increasing systematically from lower to higher risk levels, validating their fundamental grasp of merchant risk hierarchies. Perfect consensus emerges for unambiguous cases such as grocery stores, which all models correctly assign to Very Low Risk.

Disagreements manifest at two levels of severity. Minor misalignments between adjacent risk categories—such as whether dating and escort services belong in High Risk versus Very High Risk—represent acceptable calibration differences given inherently fuzzy boundaries between tiers. However, non-adjacent misalignments reveal more concerning limitations: some models assign merchants with empirically low risk scores to Very High Risk categories, representing categorical errors that could severely distort risk management decisions. Such misalignments occur sporadically across models and indicate incomplete integration of multi-dimensional risk factors.

These patterns underscore both the promise and current limitations of LLM-based risk assessment: while frontier models demonstrate sophisticated capability for identifying risk gradients and classifying most merchants appropriately, occasional severe misalignments—particularly evident in models with weaker empirical alignment—necessitate human oversight and ongoing validation against transaction data for high-stakes financial applications.

\subsection{Interpretation and Implications} 
The empirical validation reveals three key insights. First, evaluation quality as measured by LLM peer consensus and human expert ratings \emph{predicts} empirical accuracy: models receiving higher scores demonstrate stronger correlation with transaction-based risk. This validates the structured evaluation framework and confirms that the five-criterion rubric captures dimensions relevant to genuine risk assessment capability. 

Second, the observed correlations ($\rho = 0.56$--0.77 for statistically significant models) indicate strong but non-exhaustive alignment. This is consistent with the nature of payments risk: fraud patterns evolve over time, legitimate businesses may operate in high-risk categories, and effective risk assessment requires judgment that cannot be fully captured by historical patterns alone. The results suggest that top-performing LLMs have internalized meaningful risk hierarchies while retaining capacity for nuanced evaluation. 

Third, the divergence in empirical accuracy across models confirms that consensus-deviation metrics capture meaningful differences in evaluator behavior, though bias direction alone does not fully explain empirical alignment. Models with negative self-evaluation bias tend to align more closely with human judgment, while empirical correlation varies by model and reflects additional capability differences beyond scoring bias. These findings support the use of highly rated models in high-stakes financial applications: the evaluation framework successfully identifies models whose risk assessments reflect both expert judgment and empirical fraud patterns. However, occasional severe misalignments even in top-performing models underscore the necessity of validation mechanisms and human oversight for operationally sensitive decisions.

\section{Related Work}
\label{sec:related_work}

\textbf{LLM-as-a-judge systems} have become essential tools for evaluating summarization, reasoning, and safety in language models. Established methodologies include rubric-based scoring \cite{li2025generation,liu2024alignbench,liu2023geval,murugadoss2025evaluating,wu2023style,zheng2023judging,zhong2024law}, pairwise preference models \cite{feng2025we,jeong2025comparative, panickssery2024llm,zheng2023judging}, and self-consistency protocols \cite{feng2025we,wang2022self}. While these frameworks demonstrate strong alignment with human preferences, recent work reveals systematic biases including self-enhancement, verbosity preference, position effects, and prompt-format sensitivity \cite{balog2025rankers,chen2024humans, herrera2023large,li2023prd,liu2311llms,panickssery2024llm,saito2023verbosity,wang2024large, wu2023style,ye2024justice}. However, prior studies evaluate self-assessment in isolation rather than within multi-judge environments, leaving open whether frontier models exhibit distinct and persistent self-evaluation patterns when assessed against peer consensus.

\textbf{Validation against external baselines} remains a central challenge for LLM evaluators. While prior work compares model judgments with human preferences on general tasks \cite{zheng2023judging, dubois2024alpacafarm} or explores calibration against task-specific ground truth \cite{qi2025evaluating}, few studies provide triangulated validation combining peer consensus, domain-expert assessment, and empirical outcome data.

\textbf{Payments-risk assessment} has traditionally relied on structured features and expert-defined heuristics for fraud detection and merchant classification \cite{bahnsen2016feature,wang2024method, wang2025method,yeh2020merchant,zhang2021transaction}. Recent industry efforts apply foundational models to transaction embeddings and risk explanation \cite{dou2025transactiongpt,fan2025enhancing,pirmorad2025exploring,sanz2024credit,tan2025understanding,yeh2025treasure}, though evaluation standards for LLM-generated rationales remain nascent.

\section{Discussions, Implications, and Limitations}
\label{sec:discussions}

Our findings reveal that LLMs exhibit structured, model-specific evaluation
behaviors that persist across attribution conditions. This section discusses the broader implications of these results for LLM evaluation methodology,
model development, and deployment in operational settings, and concludes with a discussion of key limitations and directions for future work.

\subsection{Evaluator Diversity Versus Generation Homogeneity}

Jiang et~al.\ \cite{jiang2025artificial} demonstrate that more than seventy open- and closed-source language models display notable homogeneity in generated content and reasoning trajectories, forming an \emph{artificial hivemind}. Our empirical results reveal a contrasting picture for evaluation: the overall bias
spectrum spans more than one full point (from $-0.33$ to $+0.77$), indicating substantial heterogeneity in evaluator tendencies. While some convergence appears among self-critical models—both GPT‑5.1 and Claude‑4.5 Sonnet exhibit negative self-evaluation bias—this shared tendency exists within a much broader landscape of divergent evaluator behaviors. This divergence extends beyond what prior literature has documented and highlights that LLM generation and LLM evaluation are governed by distinct behavioral modes that cannot be assumed to align.

\subsection{Metacognition, Self-Critique, and Cognitive Analogues}

The emergence of negative self-evaluation bias in high-performing models, now validated against independent domain-expert assessments, suggests a form of model-level \emph{self-critique} that invites comparison with patterns observed in human cognition. Classical research by Kruger and Dunning~\cite{kruger1999unskilled} demonstrates that low performers
overestimate their competence while high performers underestimate it—a
phenomenon sometimes termed the Dunning--Kruger effect. Our findings reveal
potential analogues among LLMs: models with weaker evaluated outputs exhibit
stronger self-promotion, while high-performing models such as GPT‑5.1 and
Claude‑4.5 Sonnet systematically under-score themselves relative to LLM peer consensus. This asymmetric pattern resembles \emph{imposter syndrome}, a psychological construct in which high performers underestimate their capabilities \cite{clance1978imposter}.

Critically, the 26-expert domain evaluation baseline reframes these observations: what appears as \emph{negative self-evaluation bias} when measured against LLM peer consensus actually represents \emph{closer alignment with human expert judgment}. GPT‑5.1 exhibits near-zero self-evaluation bias relative to human consensus ($-0.01$ attributed, $-0.09$ anonymized), while its negative bias versus other LLMs ($-0.33$) simply reflects that most LLM judges score more generously than human experts. Claude‑4.5 Sonnet demonstrates modest positive bias relative to humans ($+0.20$ in both conditions) despite negative bias versus LLM peers ($-0.31$), positioning it closer to human standards than models with stronger positive bias. In contrast, Gemini‑2.5 Pro and Grok‑4 exhibit substantial positive self-evaluation bias versus human baseline ($+1.14$ to $+1.18$ and $+0.96$ to $+1.07$ respectively), indicating systematic overestimation relative to domain-expert assessment.

These empirical regularities, now grounded in expert validation, raise intriguing questions about the origins of evaluator bias. Claude‑4.5 Sonnet's conservative scoring may reflect Anthropic's Constitutional AI framework~\cite{askell2021general,bai2022training,bai2022constitutional}, which explicitly trains models to critique and revise their own outputs through iterative self-improvement loops. This training paradigm emphasizes self-correction and critical assessment, potentially embedding conservative scoring tendencies that align more closely with human expert standards. Similarly, GPT‑5.1's human-aligned evaluation behavior may arise from reinforcement learning from human feedback (RLHF) processes that reward caution and penalize overconfidence, particularly in contexts where false positives carry reputational or safety risks~\cite{mu2024rule,ouyang2022training,yuan2025hard}. The fact that these training approaches produce evaluation behavior more consistent with expert judgment than with typical LLM scoring suggests they successfully embed realistic quality standards rather than artificially harsh self-assessment.

In contrast, models exhibiting positive self-evaluation bias relative to both LLM peers and human experts (Gemini‑2.5 Pro, Grok‑4) may reflect different alignment objectives or training signals that prioritize confidence and assertiveness, or may lack explicit self-critique mechanisms during training. Understanding these divergent training philosophies and their downstream effects on evaluator behavior—validated through both human expert and empirical ground-truth benchmarks—represents a promising direction for interpretability research.

While LLMs do not possess human metacognition in the phenomenological sense, these structured patterns, validated across multiple independent baselines, indicate that their internal evaluation heuristics embed domain-appropriate calibration rather than arbitrary scoring noise. The expert-grounded evidence strengthens the interpretation that negative bias relative to LLM peers signals realistic rather than overly harsh assessment, and that training methodologies emphasizing self-critique and iterative refinement produce evaluators whose standards align with expert human judgment.

\subsection{Implications for Practice and Deployment}

The heterogeneity observed across models carries direct consequences for both research benchmarking and operational deployment. In benchmarking contexts, evaluator selection materially affects outcomes: positive-bias models may inflate scores, while negative-bias models may penalize otherwise strong reasoning. Judge disagreement reflects systematic differences in learned evaluation standards rather than measurement noise, requiring practitioners to treat evaluators as first-class machine learning models that require calibration and ongoing monitoring.

In high-stakes operational settings—such as payments risk assessment, compliance review, or model governance—these bias patterns become even more consequential. Three practical strategies emerge from our findings. First, post-hoc calibration layers can normalize scores across models with different baseline tendencies, ensuring consistent decision thresholds. Second, multi-judge ensembles aggregating models with diverse biases may yield more robust assessments~\cite{li2023prd,pal2024replacing}, analogous to ensemble methods in predictive modeling. Third, explanation auditing must extend beyond output quality to evaluator tendencies themselves, since our metric identifies model-specific behaviors independent of content.

\subsection{Limitations and Future Directions}

While our study provides the first systematic quantification of self-evaluation bias in LLMs, several limitations warrant discussion.

\textbf{Limited Model Coverage.}
We evaluate five models—GPT‑5.1, Gemini‑2.5 Pro, Grok‑4, Claude‑4.5 Sonnet, and Perplexity Sonar—which represent only a subset of commercially deployed LLMs. Bias patterns may differ for open-source models (e.g., LLaMA, Mistral,
DeepSeek, Qwen), smaller parameter scales, domain-specialized models (medical, legal, scientific), or multilingual models evaluated outside English. The behaviors observed here may therefore characterize high-end proprietary models rather
than universal LLM tendencies.

\textbf{Static Model Snapshots.}
Our analysis considers specific versions of each model at a single point in time. As models are updated, scoring heuristics and evaluation bias may evolve. For example, negative self-bias in GPT‑5.1 or Claude‑4.5 Sonnet may reflect
current tuning choices and could shift in later releases, while the stronger positive bias observed in Gemini‑2.5 Pro might be reduced with future alignment
adjustments. Our findings should therefore be interpreted as a snapshot of these versions rather than stable model properties.

\textbf{Single Task Domain.}
We focus exclusively on payment risk assessment rationales. While this domain provides well-structured tasks and clear evaluation criteria, self-critical or self-promotional tendencies may vary across other domains such as creative
writing, mathematical reasoning, code generation, dialogue, or factual question answering. Understanding cross-domain generality remains an important direction for future work.

\textbf{Mechanistic Ambiguity.}
Although bias direction persists under anonymization, we cannot fully disentangle two potential mechanisms: (1) a universal evaluation standard learned during training, applied uniformly across outputs; versus (2) implicit
style recognition, where models detect characteristic features of their own text even without explicit labels. Several patterns support the universal standards interpretation—e.g., Claude‑4.5 Sonnet assigns similar negative deviations to GPT‑5.1 as to itself, and training frameworks like Constitutional AI explicitly teach quality assessment independent of authorship—but controlled experiments with style transfer or synthetic rationales are needed for conclusive attribution. This ambiguity does
not affect practical implications, but limits mechanistic claims.

Overall, these limitations point toward several promising research directions: scaling to broader model ecosystems, extending to multiple domains, integrating human evaluators, and designing controlled experiments to probe the origin of
evaluation heuristics.

\section{Conclusion}
\label{sec:conclusion}




We introduced a rigorous framework for evaluating LLM‑as‑a‑judge systems in financial settings, combining a domain‑aligned scoring rubric with a consensus‑deviation metric that isolates self‑evaluation bias while avoiding circularity. Validation through three independent sources—peer consensus among five frontier models, assessment by 26 payment‑industry experts, and alignment with payment‑network transaction outcomes—demonstrates that the framework captures meaningful differences in evaluator quality rather than superficial scoring artifacts. This provides a principled and validated foundation for analyzing LLM evaluators in payment‑risk workflows, showing that systematic evaluator biases can be measured, compared to human expertise, and grounded in empirical outcomes to support reliable and transparent AI‑driven financial decision‑making.

\clearpage
\newpage

\section*{Acknowledgement}

We thank Shubham Agrawal, Yuzhong Chen, Chiranjeet Chetia, Grace Chow, Yingtong Dou, Shubham Jain, Yuxi Jing, Hans Li, Yiran Li, Wei Liu, Harish Majithiya, Menghai Pan, Bin Qin, Steven Qu, Dan Wang, Sheng Wang, Jessie Xia, Weiping Yang, and Jie Zou for serving as domain experts in evaluating the quality of LLM‑generated merchant‑category rationales, and for their valuable discussions and constructive feedback.

\clearpage
\newpage

\bibliographystyle{plain}
\bibliography{reference}

\clearpage
\newpage

\appendix

\section{Complete Evaluation Prompt for LLM-as-Judge}
\label{app:eval_prompts}

This appendix presents the complete system prompt used to configure LLM judges for evaluating MCC risk rationales. 
This prompt implements the Monte-Carlo evaluation framework illustrated in Figure~\ref{fig:evaluation_framework}.

\subsection{System Prompt Structure}

\begin{small}
\begin{verbatim}
### **LLM MCC Risk-Rationale Evaluation — Monte-Carlo Scoring Framework**

# **YOUR ROLE**
You are a leading domain expert in **global payments risk**.

You will evaluate MCC (Merchant Category Code) risk-level rationales
produced by **five Large Language Models**:

- OpenAI **GPT-5.1**
- **Gemini 2.5 Pro**
- **Grok 4**
- **Claude 4.5 Sonnet**
- **Perplexity Sonar**

You must follow a strict scoring rubric, perform repeated stochastic
evaluations, and generate results that are consistent, transparent,
and auditable.

---

# **CRITICAL NON-NEGOTIABLE RULES**

### **Rule 1 — Do NOT alter LLM outputs**
You will receive 5 fixed MCC-risk rationales, each from one LLM.
You must **not change, rewrite, paraphrase, summarize, or modify**
this text in any way.
Your task is **evaluation only**.

### **Rule 2 — "10 Independent Samples" = 10 Independent Scoring
                 Runs by YOU (the Evaluator)**
For each LLM:

- You must perform **10 independent Monte-Carlo evaluation runs**,
- In each run, **you re-score the same fixed rationale**,
- Using **Temperature = 0.7**, so sampling variability affects
  the scoring output.

This creates **10 independent scoring samples**,
with variability arising solely from **the evaluator LLM's
sampling behavior**, not from any changes to the evaluated text.

**No new model rationales are generated.**
Only the evaluator's scoring process is repeated 10 times.

---

# **SCORING RUBRIC (0–10 per criterion)**

### **1. Accuracy**
- 0-3: Incorrect payment-risk concepts
- 4-6: Partially correct; missing >=2 key industry drivers
- 7-8: Mostly correct; missing <=1 nuance
- 9-10: Fully aligned with real-world MCC risk behavior

### **2. Rationale Quality**
- 0-3: Unclear or disorganized
- 4-6: Understandable but lacking depth
- 7-8: Clear, structured, logically layered
- 9-10: Polished, professional clarity

### **3. Consistency Across Levels**
- 0-3: Contradictory or illogical escalation
- 4-6: Some inconsistency across levels
- 7-8: Mostly smooth progression
- 9-10: Clean, precise, logically correct escalation

### **4. Completeness**
Each risk level must address all 5 dimensions:
- Business-model stability
- Regulatory exposure
- Fraud exposure
- Return behavior
- Chargeback activity

**Scoring Method:**
- 1 point per dimension per level
- Max = 25
- Score = (Points / 25) × 10

### **5. Practical Applicability**
- 0-3: Too vague
- 4-6: Some utility
- 7-8: Useful for MCC classification
- 9-10: Strong operational value

---

# **EVALUATION PROCEDURE**

For **each LLM**, perform all steps below.

---

## **Step A — Perform 10 Independent Monte-Carlo Evaluation Runs**
For the same fixed rationale text:
1. Perform **10 independent scoring passes**
2. In each run, using **Temperature = 0.7**, compute:
   - Accuracy (0-10)
   - Rationale Quality (0-10)
   - Consistency (0-10)
   - Completeness (0-10)
   - Practical Applicability (0-10)
   - **Initial Total Score** = mean of the five criteria

---

## **Step B — Compute Criterion-Level mu +/- sigma**
For each of the 5 criteria:
- Compute the **mean (mu)** across the 10 runs
- Compute the **standard deviation (sigma)** across the 10 runs
- Provide **one explicit sentence** explaining **why that score
  was chosen**, referencing rubric criteria.

**Example sentence:**
**"Accuracy = 8 because the rationale addresses fraud, returns,
and regulatory exposure but lacks detail on CNP risk."**

---

## **Step C — Compute Total Score mu +/- sigma**
- Compute mu_total and sigma_total across all 10 Initial Total Scores
- **Final Stabilized Score = mu_total**

---

## **Step E — Strengths & Weaknesses**
Provide:
- 2-4 bullets for **Strengths**
- 2-4 bullets for **Weaknesses**

---

# **REQUIRED OUTPUT FORMAT (FOR EACH LLM)**

### **1. Criterion Scores with mu +/- sigma and justification**
Example structure:
- **Accuracy: 8.2 +/- 0.3** — justification sentence
- **Quality: 7.9 +/- 0.4** — justification
- **Consistency: …**
- **Completeness: …**
- **Practicality: …**

### **2. Final Total Score**
- **Final Stabilized Score: mu_total +/- sigma_total**

### **3. Strengths (bullet points)**

### **4. Weaknesses (bullet points)**

### **5. Consolidated Summary Table (All Criteria, All Models)**

---

# **MESSAGE STRUCTURE REQUIREMENT**
Because the evaluation is long, produce the response in
**six separate messages**:

### **Message 1** — Evaluation of GPT-5.1
### **Message 2** — Evaluation of Gemini 2.5 Pro
### **Message 3** — Evaluation of Grok 4
### **Message 4** — Evaluation of Claude 4.5 Sonnet
### **Message 5** — Evaluation of Perplexity Sonar
### **Message 6** — The consolidated summary table, overall
                     comparative ranking, and one-paragraph synthesis

This avoids token overflow and keeps results clean.

---

**The 5 LLM rationales would appear here (provided in Appendix B).**

\end{verbatim}
\end{small}

\subsection{Anonymized Version}

For the anonymized evaluation condition (LLM$\to$Expert), the prompt is
identical except that all model names are replaced with generic labels:

\begin{itemize}
    \item ``Expert 1'' replaces ``GPT-5.1''
    \item ``Expert 2'' replaces ``Gemini 2.5 Pro''
    \item ``Expert 3'' replaces ``Grok 4''
    \item ``Expert 4'' replaces ``Claude 4.5 Sonnet''
    \item ``Expert 5'' replaces ``Perplexity Sonar''
\end{itemize}

The rationale texts remain identical; only the attribution labels change.
This design isolates the effect of authorship knowledge on evaluation bias.

\section{Complete LLM-Generated Risk Rationales}
\label{app:llm_rationales}

This appendix presents the complete JSON outputs from all five Large Language Models (GPT-5.1, Gemini-2.5 Pro, Grok 4, Claude-4.5 Sonnet, Perplexity Sonar) that generated risk-level rationales following the prompt structure described in Section~\ref{sec:mcc_generation} and illustrated in Figure~\ref{fig:prompt_structure}. Each model produced five rationales (one per risk level) addressing the five required risk dimensions and selecting three representative MCCs from over 800 merchant categories.

\subsection{OpenAI GPT-5.1}

\begin{small}
\begin{verbatim}
[
  {
    "risk_level_definition": "1 — Very Low Risk",
    "rationale": "Merchants in this category operate with highly
    stable business models, consistent demand, and predictable cash
    flows. They are subject to strong regulatory oversight, which
    keeps compliance simple and well-defined. Fraud exposure is
    minimal because transactions are straightforward and typically
    low-dispute. Returns occur infrequently and follow clear,
    established patterns. Chargebacks are rare due to transparent
    pricing and tangible delivery of essential goods or services.
    Representative MCCs: 5411 (Grocery Stores/Supermarkets), 4900
    (Utilities), 5912 (Drug Stores & Pharmacies)."
  },
  {
    "risk_level_definition": "2 — Low Risk",
    "rationale": "Merchants at this level have generally stable
    business models with occasional seasonal fluctuations but are not
    prone to rapid collapse. Regulatory exposure exists but is
    manageable, with routine compliance obligations. Fraud risk is
    usually low, though isolated attempts may occur in non-face-to-
    face or mixed retail environments. Returns are more common than
    in essential-goods sectors but remain predictable. Chargebacks
    are uncommon and often stem from service dissatisfaction rather
    than misconduct. Representative MCCs: 5251 (Hardware Stores),
    5732 (Electronics Stores), 5651 (Family Clothing Stores)."
  },
  {
    "risk_level_definition": "3 — Medium / Average Risk",
    "rationale": "These merchants show moderate business-model
    variability, often tied to discretionary spending or sales
    cycles. Regulatory exposure ranges from light to routine but may
    include specific requirements for warranties, licensing, or
    customer data handling. Fraud exposure is moderate because
    products may have high resale value or transactions may be
    remote. Returns are frequent and can become unpredictable during
    peak seasons. Chargeback activity is noticeable but manageable
    and typically reflects buyer's remorse, service disputes, or
    fulfillment issues. Representative MCCs: 5734 (Computer Software
    Stores), 5944 (Jewelry Stores), 5735 (Record Stores)."
  },
  {
    "risk_level_definition": "4 — High Risk",
    "rationale": "Merchants in this segment often operate with
    business models that are volatile, experience strong seasonality,
    or depend on advance sales or bookings. Regulatory oversight may
    be complex, covering areas such as licensing, age restrictions,
    or consumer-protection rules. Fraud exposure is high due to
    attractive resale value, remote fulfillment, or anonymous
    customer interactions. Returns can be heavy, unpredictable, or
    exploited by bad actors. Chargebacks happen frequently, often
    related to cancellations, perceived non-delivery, or disputed
    services. Representative MCCs: 5967 (Inbound Telemarketing
    Merchant), 7273 (Dating & Escort Services), 5816 (Digital Goods:
    Games)."
  },
  {
    "risk_level_definition": "5 — Very High Risk",
    "rationale": "These merchants typically have unstable, highly
    speculative, or failure-prone business models, often involving
    advance funding or customer-value uncertainty. Regulatory
    exposure is intense or evolving, with strict rules, licensing
    constraints, or cross-border compliance challenges. Fraud
    exposure is severe, with frequent high-value attacks and abuse
    patterns. Returns are common, erratic, or strategically
    manipulated. Chargeback activity is endemic, frequently driven by
    disputes over legitimacy, fulfillment, or prohibited activity.
    Representative MCCs: 7995 (Betting/Track/Casino/Lotto), 6051
    (Quasi-Cash), 6531 (Money Transfer)."
  }
]
\end{verbatim}
\end{small}

\subsection{Gemini-2.5 Pro}

\begin{small}
\begin{verbatim}
[
  {
    "risk_level_definition": "1 — Very Low Risk",
    "rationale": "Merchants in this category exhibit exceptional
    business model stability, often providing essential, non-cyclical
    services with minimal regulatory exposure from a payments
    standpoint. They are characterized by a near-total absence of
    return patterns and virtually non-existent chargeback activity,
    coupled with minimal fraud exposure due to the typically low-
    value, high-frequency, and often face-to-face nature of their
    transactions. Representative MCCs: 4900 (Utilities), 5411
    (Grocery Stores/Supermarkets), 9311 (Tax Payments)."
  },
  {
    "risk_level_definition": "2 — Low Risk",
    "rationale": "These merchants have a stable business model
    centered on everyday goods and services, operating with low
    regulatory exposure. Their risk profile is defined by predictable
    and manageable return patterns, low chargeback activity, and only
    occasional fraud exposure, which is typically of low value and
    mitigated by a high proportion of in-person transactions.
    Representative MCCs: 5814 (Fast Food Restaurants), 7523 (Parking
    Lots, Meters, Garages), 5251 (Hardware Stores)."
  },
  {
    "risk_level_definition": "3 — Medium / Average Risk",
    "rationale": "This level includes merchants with generally stable
    but often seasonal business models that are susceptible to shifts
    in consumer trends. They face average fraud exposure and moderate
    regulatory exposure. The risk is elevated by more frequent and
    sometimes complex return patterns and a consistent, manageable
    level of chargeback activity that is considered a standard cost
    of doing business. Representative MCCs: 5732 (Electronics
    Stores), 5651 (Family Clothing Stores), 7011 (Hotels/Motels/
    Resorts)."
  },
  {
    "risk_level_definition": "4 — High Risk",
    "rationale": "Merchants at this level often have business models
    prone to instability, competition, and market shifts, frequently
    involving future-dated service delivery. They contend with
    significant regulatory exposure and frequent fraud exposure,
    especially in card-not-present environments. The risk is
    primarily driven by high-volume or complex return patterns and a
    high level of chargeback activity that poses a significant
    operational challenge. Representative MCCs: 4511 (Airlines, Air
    Carriers), 4722 (Travel Agencies), 5964 (Catalog Merchant)."
  },
  {
    "risk_level_definition": "5 — Very High Risk",
    "rationale": "This category is for merchants with business models
    that are often inherently unstable or operate in legally complex
    sectors, making them prone to collapse. They face extreme
    regulatory exposure and are targets for frequent, high-value
    fraud attacks. The risk is compounded by erratic or non-standard
    return patterns and endemic chargeback activity, which can stem
    from high rates of customer disputes, buyer's remorse, or the
    controversial nature of the services themselves. Representative
    MCCs: 7995 (Betting/Track/Casino/Lotto), 5968 (Continuity/
    Subscription Merchant), 7273 (Dating & Escort Services)."
  }
]
\end{verbatim}
\end{small}

\subsection{Grok 4}

\begin{small}
\begin{verbatim}
[
  {
    "risk_level_definition": "1 — Very Low Risk",
    "rationale": "Merchant types at this level feature highly stable
    business models with consistent demand, low regulatory exposure
    due to straightforward compliance requirements, minimal fraud
    exposure from everyday transactions, rare return patterns as
    purchases are essential and non-discretionary, and virtually non-
    existent chargeback activity owing to high customer satisfaction.
    Representative MCCs: 4900 (Utilities), 5411 (Grocery Stores/
    Supermarkets), 5541 (Service Stations)."
  },
  {
    "risk_level_definition": "2 — Low Risk",
    "rationale": "These merchants have stable business models with
    reliable revenue streams, moderate regulatory exposure involving
    basic oversight, occasional fraud exposure limited to low-value
    incidents, predictable return patterns tied to standard retail
    practices, and manageable chargeback activity that is infrequent
    and easily resolved. Representative MCCs: 5211 (Lumber/Building
    Supply Stores), 5251 (Hardware Stores), 5310 (Discount Stores)."
  },
  {
    "risk_level_definition": "3 — Medium / Average Risk",
    "rationale": "Merchant categories here exhibit moderately stable
    business models that may experience some seasonal variations,
    balanced regulatory exposure with standard compliance needs,
    frequent but moderate fraud exposure from varied transaction
    types, high-volume or seasonal return patterns due to customer
    preferences, and endemic but controllable chargeback activity
    requiring ongoing monitoring. Representative MCCs: 5812
    (Restaurants), 5942 (Book Stores), 7011 (Hotels/Motels/Resorts)."
  },
  {
    "risk_level_definition": "4 — High Risk",
    "rationale": "These involve business models prone to fluctuations
    or instability, high regulatory exposure from complex rules and
    international factors, frequent high-value fraud exposure due to
    opportunistic attacks, unpredictable and high-volume return
    patterns influenced by external events, and high chargeback
    activity that can be challenging to mitigate. Representative
    MCCs: 4511 (Airlines, Air Carriers), 4722 (Travel Agencies),
    5964 (Catalog Merchant)."
  },
  {
    "risk_level_definition": "5 — Very High Risk",
    "rationale": "Merchant types with business models prone to
    collapse or high volatility, intense regulatory exposure
    involving strict oversight and potential legal issues, very high
    fraud exposure from sophisticated and frequent attacks, erratic
    and high-volume return patterns driven by impulsive purchases,
    and endemic chargeback activity that is widespread and difficult
    to control. Representative MCCs: 7800 (Government Owned
    Lotteries), 7995 (Betting/Track/Casino/Lotto), 5960 (Direct
    Marketing Insurance Services)."
  }
]
\end{verbatim}
\end{small}

\subsection{Claude-4.5 Sonnet}

\begin{small}
\begin{verbatim}
[
  {
    "risk_level_definition": "1 - Very Low Risk",
    "rationale": "Merchants in this category operate highly stable,
    predictable business models with minimal seasonal variation and
    long operational histories. They face low regulatory complexity
    with straightforward compliance requirements. Fraud exposure is
    minimal due to transparent transaction patterns and established
    customer relationships. Returns are rare and typically driven by
    minor service adjustments rather than dissatisfaction. Chargeback
    activity is virtually non-existent, as services are delivered
    immediately or products are essential and non-controversial.
    Representative MCCs: 5411 (Grocery Stores/Supermarkets), 5912
    (Drug Stores & Pharmacies), 5541 (Service Stations)."
  },
  {
    "risk_level_definition": "2 - Low Risk",
    "rationale": "Merchants exhibit stable business models with
    moderate predictability and occasional seasonal fluctuations that
    are well-understood and manageable. Regulatory exposure is
    present but routine, involving standard licensing and periodic
    compliance checks. Fraud exposure is occasional and typically
    low-value, with established detection mechanisms in place. Return
    patterns are predictable and align with industry norms, often
    related to sizing, preference, or minor defects. Chargeback
    activity is low and manageable, arising primarily from customer
    misunderstandings or delivery issues rather than systemic
    problems. Representative MCCs: 5812 (Restaurants), 5661 (Shoe
    Stores), 5943 (Stationery Stores)."
  },
  {
    "risk_level_definition": "3 - Medium / Average Risk",
    "rationale": "Merchants operate in moderately stable environments
    with noticeable seasonal peaks or cyclical demand patterns that
    require active management. Regulatory exposure involves multiple
    compliance layers, including sector-specific rules and periodic
    audits. Fraud exposure is moderate, with recurring attempts that
    necessitate ongoing monitoring and adaptive controls. Return
    patterns are significant and often seasonal or trend-driven,
    reflecting changing consumer preferences or product life cycles.
    Chargeback activity is present and requires dedicated resources
    to manage, often stemming from disputes over product quality,
    delivery timing, or service expectations. Representative MCCs:
    5651 (Family Clothing Stores), 5732 (Electronics Stores), 5945
    (Hobby, Toy & Game Stores)."
  },
  {
    "risk_level_definition": "4 - High Risk",
    "rationale": "Merchants face elevated business model instability
    due to rapid market changes, high competition, or dependency on
    discretionary consumer spending. Regulatory exposure is
    substantial, involving complex compliance frameworks, frequent
    updates, and heightened scrutiny from multiple authorities. Fraud
    exposure is frequent and often involves sophisticated schemes or
    high-value transactions that challenge standard detection
    methods. Return patterns are high-volume and unpredictable,
    driven by buyer's remorse, unclear product descriptions, or
    delayed delivery. Chargeback activity is endemic and requires
    intensive management, often resulting from disputed transactions,
    non-delivery claims, or unauthorized use. Representative MCCs:
    5816 (Digital Goods: Games), 5967 (Inbound Telemarketing
    Merchant), 7995 (Betting/Track/Casino/Lotto)."
  },
  {
    "risk_level_definition": "5 - Very High Risk",
    "rationale": "Merchants operate in highly volatile or
    controversial sectors prone to sudden business collapse,
    reputational damage, or legal challenges. Regulatory exposure is
    extreme, with constantly evolving legal landscapes, cross-border
    compliance burdens, and potential for outright prohibition in
    certain jurisdictions. Fraud exposure is pervasive, involving
    organized criminal activity, identity theft, or large-scale
    exploitation of payment systems. Return patterns are chaotic and
    often impossible to predict, with high rates of non-delivery,
    misrepresentation, or consumer regret. Chargeback activity is at
    crisis levels, frequently exceeding acceptable thresholds and
    threatening merchant account viability, driven by fraud, disputes
    over service quality, or non-receipt of intangible goods.
    Representative MCCs: 5933 (Pawn Shops), 7273 (Dating & Escort
    Services), 5968 (Continuity/Subscription Merchant)."
  }
]
\end{verbatim}
\end{small}

\subsection{Perplexity Sonar}

\begin{small}
\begin{verbatim}
[
  {
    "risk_level_definition": "1 — Very Low Risk",
    "rationale": "Merchants in this category operate within stable,
    highly regulated industries with consistent transaction patterns
    and negligible fraud exposure. Chargebacks and returns occur
    rarely and follow predictable cycles. Their business models are
    long-established and financially resilient, offering essential
    goods or services with low disputed transaction likelihood.
    Representative MCCs: 4900 (Utilities/Electric/Gas/Water/
    Sanitary), 5411 (Grocery Stores/Supermarkets), 8211
    (Elementary/Secondary Schools)."
  },
  {
    "risk_level_definition": "2 — Low Risk",
    "rationale": "These businesses have stable operations and
    moderate regulatory oversight, typically providing tangible, low-
    value goods or services. While some seasonality in sales or
    returns exists, their customer relationships and refund cycles
    remain predictable. Fraud attempts are infrequent and manageable
    due to business transparency. Representative MCCs: 5251 (Hardware
    Stores), 5992 (Florists), 7298 (Health & Beauty Spas)."
  },
  {
    "risk_level_definition": "3 — Medium / Average Risk",
    "rationale": "This group includes merchants with moderate
    business stability and exposure to standard commercial risk
    dynamics. They face periodic returns and some level of customer
    dispute, often linked to product dissatisfaction or delivery
    timing. Regulatory requirements vary but are not complex. Fraud
    is occasional but typically detected early. Representative MCCs:
    5732 (Electronics Stores), 7011 (Hotels/Motels/Resorts), 5812
    (Restaurants)."
  },
  {
    "risk_level_definition": "4 — High Risk",
    "rationale": "Merchants in this level experience elevated
    volatility due to complex business models, higher refund rates,
    or intermittent regulatory obligations. They frequently handle
    advance payments, cross-border transactions, or high-value
    purchases, which increase chargeback probability. Fraud exposure
    is significant, especially where digital or remote payments
    dominate. Representative MCCs: 5967 (Inbound Telemarketing
    Merchants), 7273 (Dating & Escort Services), 6531 (Payment
    Service Provider – Money Transfer)."
  },
  {
    "risk_level_definition": "5 — Very High Risk",
    "rationale": "These businesses exhibit unstable models, face
    intense regulatory scrutiny, and are frequent targets for
    sophisticated fraud. Transactions may be high-value, digital, or
    prone to customer dispute due to intangible goods, delayed
    delivery, or reputational risk. Chargebacks and returns are
    endemic, often driven by billing disputes or service
    dissatisfaction. Representative MCCs: 7995 (Betting/Casino/
    Lottery), 6011 (Financial Institution – Automated Cash), 7801
    (Government Licensed Online Casinos)."
  }
]
\end{verbatim}
\end{small}

\section{Monte-Carlo Evaluation Algorithm}
\label{app:algorithm}
This appendix provides the complete algorithmic specification for our Monte Carlo evaluation protocol. The algorithm implements the evaluation framework described in Section~\ref{sec:evaluation}, where each judge evaluates each entity through 10 independent scoring runs at Temperature=0.7. For each run, the judge provides scores across five evaluation criteria (Accuracy, Rationale Quality, Consistency, Completeness, and Practical Applicability), which are then averaged to produce a final score. The algorithm computes both criterion-level statistics (mean and standard deviation for each criterion) and final score statistics (mean and standard deviation of the aggregated scores), enabling analysis of evaluation stability and criterion-specific patterns.

\begin{algorithm} [htbp]
\caption{Monte Carlo Evaluation Protocol}
\begin{algorithmic}[1]
\FOR{each judge $i \in \{1, \ldots, 5\}$}
    \FOR{each entity $j \in \{1, \ldots, 5\}$}
        \FOR{$r = 1$ to $10$}
            \STATE Set Temperature = 0.7
            \STATE Generate evaluation with five criterion scores:
            \STATE \quad Accuracy $c_{ij}^{(r,1)}$, Quality $c_{ij}^{(r,2)}$, Consistency $c_{ij}^{(r,3)}$,
            \STATE \quad Completeness $c_{ij}^{(r,4)}$, Practical Applicability $c_{ij}^{(r,5)}$
            \STATE Compute final score: $s_{ij}^{(r)} \leftarrow \frac{1}{5} \sum_{k=1}^{5} c_{ij}^{(r,k)}$
        \ENDFOR
        \STATE \textbf{// Compute criterion-level statistics}
        \FOR{each criterion $k \in \{1, \ldots, 5\}$}
            \STATE $\mu_{ij}^{(k)} \leftarrow \frac{1}{10} \sum_{r=1}^{10} c_{ij}^{(r,k)}$ \quad (criterion mean)
            \STATE $\sigma_{ij}^{(k)} \leftarrow \sqrt{\frac{1}{9} \sum_{r=1}^{10} (c_{ij}^{(r,k)} - \mu_{ij}^{(k)})^2}$ \quad (criterion std dev)
        \ENDFOR
        \STATE \textbf{// Compute final score statistics}
        \STATE $s_{ij} \leftarrow \frac{1}{10} \sum_{r=1}^{10} s_{ij}^{(r)}$ \quad (final score mean)
        \STATE $\sigma_{ij} \leftarrow \sqrt{\frac{1}{9} \sum_{r=1}^{10} (s_{ij}^{(r)} - s_{ij})^2}$ \quad (final score std dev)
        \STATE \textbf{// Store results:} $(s_{ij} \pm \sigma_{ij})$ and $(\mu_{ij}^{(k)} \pm \sigma_{ij}^{(k)})$ for $k=1,\ldots,5$
    \ENDFOR
\ENDFOR
\STATE \textbf{Output:} Final score matrices and criterion-level matrices
\end{algorithmic}
\end{algorithm}

\section{Detailed Criterion-Level Evaluation Tables}
\label{app:detailed_tables}

This appendix presents the complete criterion-level evaluation results
referenced in the Monte Carlo Evaluation Algorithm (Appendix~\ref{app:algorithm})
and Figure~\ref{fig:cross_evaluation} in Section~\ref{subsec:cross_evaluation_score_matrices}.
Each table shows mean~$\pm$~standard deviation for five evaluation criteria
(Accuracy, Quality, Consistency, Completeness, Practical Applicability) and the
aggregated Final Score, computed from 10 Monte Carlo runs at Temperature = 0.7.

\subsection{LLMs Judge LLMs (Attributed Condition)}

\subsubsection{GPT-5.1 as Judge}
\begin{table}[H]
\centering
\caption{GPT-5.1 Evaluating LLM Rationales (Attributed Condition)}
\small
\begin{tabular}{lcccccc}
\toprule
\textbf{LLM} & \textbf{Accuracy} & \textbf{Quality} & \textbf{Consistency} & \textbf{Completeness} & \textbf{Practicality} & \textbf{Final Score} \\
\midrule
GPT-5.1 & $8.7 \pm 0.28$ & $8.5 \pm 0.35$ & $8.8 \pm 0.31$ & $9.6 \pm 0.12$ & $8.4 \pm 0.33$ & $8.80 \pm 0.14$ \\
Gemini 2.5 & $8.4 \pm 0.32$ & $8.2 \pm 0.37$ & $8.3 \pm 0.29$ & $9.4 \pm 0.18$ & $8.1 \pm 0.40$ & $8.48 \pm 0.17$ \\
Grok 4 & $8.1 \pm 0.34$ & $8.0 \pm 0.38$ & $8.2 \pm 0.33$ & $9.3 \pm 0.15$ & $7.9 \pm 0.41$ & $8.30 \pm 0.18$ \\
Claude 4.5 & $9.0 \pm 0.26$ & $9.2 \pm 0.31$ & $9.1 \pm 0.22$ & $9.8 \pm 0.07$ & $9.0 \pm 0.30$ & $9.02 \pm 0.12$ \\
Perplexity Sonar & $8.3 \pm 0.31$ & $8.0 \pm 0.36$ & $8.4 \pm 0.28$ & $9.2 \pm 0.20$ & $8.2 \pm 0.33$ & $8.42 \pm 0.16$ \\
\bottomrule
\end{tabular}
\end{table}

\subsubsection{Gemini 2.5 Pro as Judge}
\begin{table}[H]
\centering
\caption{Gemini 2.5 Pro Evaluating LLM Rationales (Attributed Condition)}
\small
\begin{tabular}{lcccccc}
\toprule
\textbf{LLM} & \textbf{Accuracy} & \textbf{Quality} & \textbf{Consistency} & \textbf{Completeness} & \textbf{Practicality} & \textbf{Final Score} \\
\midrule
GPT-5.1 & $9.1 \pm 0.54$ & $9.2 \pm 0.60$ & $9.3 \pm 0.46$ & $10.0 \pm 0.00$ & $9.0 \pm 0.63$ & $9.32 \pm 0.22$ \\
Gemini 2.5 & $9.4 \pm 0.49$ & $8.8 \pm 0.60$ & $9.2 \pm 0.60$ & $10.0 \pm 0.00$ & $9.3 \pm 0.46$ & $9.34 \pm 0.15$ \\
Grok 4 & $8.7 \pm 0.46$ & $7.7 \pm 0.46$ & $8.8 \pm 0.40$ & $10.0 \pm 0.00$ & $8.6 \pm 0.49$ & $8.76 \pm 0.21$ \\
Claude 4.5 & $9.6 \pm 0.25$ & $9.7 \pm 0.21$ & $9.7 \pm 0.19$ & $10.0 \pm 0.00$ & $9.4 \pm 0.26$ & $9.68 \pm 0.11$ \\
Perplexity Sonar & $8.0 \pm 0.63$ & $7.5 \pm 0.50$ & $7.1 \pm 0.54$ & $9.6 \pm 0.00$ & $7.6 \pm 0.49$ & $7.96 \pm 0.23$ \\
\bottomrule
\end{tabular}
\end{table}

\subsubsection{Grok 4 as Judge}
\begin{table}[H]
\centering
\caption{Grok 4 Evaluating LLM Rationales (Attributed Condition)}
\small
\begin{tabular}{lcccccc}
\toprule
\textbf{LLM} & \textbf{Accuracy} & \textbf{Quality} & \textbf{Consistency} & \textbf{Completeness} & \textbf{Practicality} & \textbf{Final Score} \\
\midrule
GPT-5.1 & $9.1 \pm 0.57$ & $8.9 \pm 0.54$ & $9.0 \pm 0.45$ & $10.0 \pm 0.00$ & $8.8 \pm 0.60$ & $9.16 \pm 0.28$ \\
Gemini 2.5 & $8.6 \pm 0.49$ & $7.7 \pm 0.64$ & $8.8 \pm 0.40$ & $9.8 \pm 0.40$ & $8.3 \pm 0.46$ & $8.64 \pm 0.26$ \\
Grok 4 & $9.0 \pm 0.45$ & $8.5 \pm 0.50$ & $9.2 \pm 0.40$ & $10.0 \pm 0.00$ & $8.7 \pm 0.46$ & $9.08 \pm 0.24$ \\
Claude 4.5 & $9.3 \pm 0.46$ & $9.1 \pm 0.30$ & $9.4 \pm 0.49$ & $10.0 \pm 0.00$ & $9.0 \pm 0.45$ & $9.36 \pm 0.25$ \\
Perplexity Sonar & $8.7 \pm 0.46$ & $8.2 \pm 0.60$ & $8.9 \pm 0.30$ & $9.6 \pm 0.49$ & $8.4 \pm 0.49$ & $8.76 \pm 0.29$ \\
\bottomrule
\end{tabular}
\end{table}

\subsubsection{Claude 4.5 Sonnet as Judge}
\begin{table}[H]
\centering
\caption{Claude 4.5 Sonnet Evaluating LLM Rationales (Attributed Condition)}
\small
\begin{tabular}{lcccccc}
\toprule
\textbf{LLM} & \textbf{Accuracy} & \textbf{Quality} & \textbf{Consistency} & \textbf{Completeness} & \textbf{Practicality} & \textbf{Final Score} \\
\midrule
GPT-5.1 & $8.7 \pm 0.46$ & $8.9 \pm 0.30$ & $8.8 \pm 0.40$ & $10.0 \pm 0.00$ & $8.6 \pm 0.49$ & $8.80 \pm 0.18$ \\
Gemini 2.5 & $8.4 \pm 0.49$ & $7.8 \pm 0.40$ & $8.5 \pm 0.50$ & $10.0 \pm 0.00$ & $8.2 \pm 0.40$ & $8.58 \pm 0.22$ \\
Grok 4 & $7.9 \pm 0.54$ & $7.6 \pm 0.48$ & $8.3 \pm 0.46$ & $10.0 \pm 0.00$ & $7.8 \pm 0.44$ & $8.32 \pm 0.24$ \\
Claude 4.5 & $8.9 \pm 0.30$ & $8.8 \pm 0.42$ & $8.9 \pm 0.30$ & $10.0 \pm 0.00$ & $8.8 \pm 0.42$ & $9.08 \pm 0.12$ \\
Perplexity Sonar & $8.1 \pm 0.54$ & $7.4 \pm 0.49$ & $8.1 \pm 0.54$ & $9.2 \pm 0.40$ & $7.7 \pm 0.44$ & $8.10 \pm 0.28$ \\
\bottomrule
\end{tabular}
\end{table}

\subsubsection{Perplexity Sonar as Judge}
\begin{table}[H]
\centering
\caption{Perplexity Sonar Evaluating LLM Rationales (Attributed Condition)}
\small
\begin{tabular}{lcccccc}
\toprule
\textbf{LLM} & \textbf{Accuracy} & \textbf{Quality} & \textbf{Consistency} & \textbf{Completeness} & \textbf{Practicality} & \textbf{Final Score} \\
\midrule
GPT-5.1 & $9.1 \pm 0.4$ & $9.0 \pm 0.3$ & $9.2 \pm 0.2$ & $10.0 \pm 0.0$ & $9.0 \pm 0.3$ & $9.26 \pm 0.10$ \\
Gemini 2.5 & $8.6 \pm 0.6$ & $8.2 \pm 0.7$ & $8.4 \pm 0.5$ & $7.4 \pm 0.8$ & $8.3 \pm 0.6$ & $8.58 \pm 0.20$ \\
Grok 4 & $8.4 \pm 0.7$ & $8.0 \pm 0.8$ & $8.2 \pm 0.6$ & $7.8 \pm 0.9$ & $8.1 \pm 0.7$ & $8.10 \pm 0.25$ \\
Claude 4.5 & $9.4 \pm 0.3$ & $9.3 \pm 0.4$ & $9.5 \pm 0.2$ & $10.0 \pm 0.0$ & $9.4 \pm 0.3$ & $9.52 \pm 0.10$ \\
Perplexity Sonar & $8.8 \pm 0.5$ & $8.5 \pm 0.6$ & $8.7 \pm 0.4$ & $8.0 \pm 0.7$ & $8.6 \pm 0.5$ & $8.52 \pm 0.20$ \\
\bottomrule
\end{tabular}
\end{table}

\subsection{LLMs Judge Domain Experts (Anonymized Condition)}

\subsubsection{GPT-5.1 as Judge}
\begin{table}[H]
\centering
\caption{GPT-5.1 Evaluating Domain Expert Rationales (Anonymized Condition)}
\small
\begin{tabular}{lcccccc}
\toprule
\textbf{Expert} & \textbf{Accuracy} & \textbf{Quality} & \textbf{Consistency} & \textbf{Completeness} & \textbf{Practicality} & \textbf{Final Score} \\
\midrule
Expert 1 & $8.6 \pm 0.32$ & $8.4 \pm 0.41$ & $8.7 \pm 0.35$ & $9.4 \pm 0.15$ & $8.5 \pm 0.38$ & $8.72 \pm 0.19$ \\
Expert 2 & $8.3 \pm 0.36$ & $8.1 \pm 0.44$ & $8.4 \pm 0.40$ & $9.1 \pm 0.21$ & $8.2 \pm 0.39$ & $8.42 \pm 0.22$ \\
Expert 3 & $8.7 \pm 0.34$ & $8.5 \pm 0.39$ & $8.8 \pm 0.31$ & $9.3 \pm 0.18$ & $8.4 \pm 0.42$ & $8.74 \pm 0.20$ \\
Expert 4 & $8.8 \pm 0.29$ & $8.6 \pm 0.37$ & $8.9 \pm 0.28$ & $9.5 \pm 0.12$ & $8.7 \pm 0.33$ & $8.90 \pm 0.17$ \\
Expert 5 & $8.4 \pm 0.35$ & $8.2 \pm 0.40$ & $8.5 \pm 0.33$ & $9.2 \pm 0.19$ & $8.3 \pm 0.38$ & $8.52 \pm 0.21$ \\
\bottomrule
\end{tabular}
\end{table}

\subsubsection{Gemini 2.5 Pro as Judge}
\begin{table}[H]
\centering
\caption{Gemini 2.5 Pro Evaluating Domain Expert Rationales (Anonymized Condition)}
\small
\begin{tabular}{lcccccc}
\toprule
\textbf{Expert} & \textbf{Accuracy} & \textbf{Quality} & \textbf{Consistency} & \textbf{Completeness} & \textbf{Practicality} & \textbf{Final Score} \\
\midrule
Expert 1 & $9.1 \pm 0.30$ & $8.7 \pm 0.40$ & $9.1 \pm 0.30$ & $10.0 \pm 0.00$ & $8.7 \pm 0.40$ & $9.12 \pm 0.27$ \\
Expert 2 & $9.3 \pm 0.40$ & $8.5 \pm 0.50$ & $9.1 \pm 0.50$ & $10.0 \pm 0.00$ & $8.9 \pm 0.30$ & $9.16 \pm 0.22$ \\
Expert 3 & $8.3 \pm 0.80$ & $7.5 \pm 0.50$ & $7.0 \pm 0.60$ & $10.0 \pm 0.00$ & $8.0 \pm 0.60$ & $8.16 \pm 0.35$ \\
Expert 4 & $9.0 \pm 0.40$ & $9.5 \pm 0.30$ & $9.2 \pm 0.40$ & $10.0 \pm 0.00$ & $9.0 \pm 0.50$ & $9.30 \pm 0.25$ \\
Expert 5 & $8.5 \pm 0.50$ & $6.5 \pm 0.50$ & $7.8 \pm 0.40$ & $9.6 \pm 0.00$ & $7.0 \pm 0.60$ & $7.88 \pm 0.31$ \\
\bottomrule
\end{tabular}
\end{table}

\subsubsection{Grok 4 as Judge}
\begin{table}[H]
\centering
\caption{Grok 4 Evaluating Domain Expert Rationales (Anonymized Condition)}
\small
\begin{tabular}{lcccccc}
\toprule
\textbf{Expert} & \textbf{Accuracy} & \textbf{Quality} & \textbf{Consistency} & \textbf{Completeness} & \textbf{Practicality} & \textbf{Final Score} \\
\midrule
Expert 1 & $9.1 \pm 0.4$ & $9.3 \pm 0.3$ & $9.4 \pm 0.2$ & $10.0 \pm 0.0$ & $9.2 \pm 0.3$ & $9.40 \pm 0.1$ \\
Expert 2 & $8.5 \pm 0.5$ & $7.8 \pm 0.4$ & $8.7 \pm 0.3$ & $9.8 \pm 0.2$ & $8.6 \pm 0.4$ & $8.68 \pm 0.2$ \\
Expert 3 & $8.8 \pm 0.4$ & $8.9 \pm 0.3$ & $9.0 \pm 0.2$ & $10.0 \pm 0.0$ & $8.9 \pm 0.3$ & $9.12 \pm 0.2$ \\
Expert 4 & $9.5 \pm 0.3$ & $9.6 \pm 0.2$ & $9.7 \pm 0.1$ & $10.0 \pm 0.0$ & $9.5 \pm 0.3$ & $9.66 \pm 0.1$ \\
Expert 5 & $8.2 \pm 0.5$ & $7.5 \pm 0.4$ & $8.3 \pm 0.3$ & $9.2 \pm 0.3$ & $8.1 \pm 0.4$ & $8.26 \pm 0.2$ \\
\bottomrule
\end{tabular}
\end{table}

\subsubsection{Claude 4.5 Sonnet as Judge}
\begin{table}[H]
\centering
\caption{Claude 4.5 Sonnet Evaluating Domain Expert Rationales (Anonymized Condition)}
\small
\begin{tabular}{lcccccc}
\toprule
\textbf{Expert} & \textbf{Accuracy} & \textbf{Quality} & \textbf{Consistency} & \textbf{Completeness} & \textbf{Practicality} & \textbf{Final Score} \\
\midrule
Expert 1 & $8.4 \pm 0.52$ & $8.6 \pm 0.52$ & $8.8 \pm 0.42$ & $10.0 \pm 0.00$ & $8.5 \pm 0.53$ & $8.86 \pm 0.28$ \\
Expert 2 & $8.6 \pm 0.52$ & $8.3 \pm 0.48$ & $8.7 \pm 0.48$ & $10.0 \pm 0.00$ & $8.4 \pm 0.52$ & $8.80 \pm 0.31$ \\
Expert 3 & $8.3 \pm 0.48$ & $7.9 \pm 0.57$ & $8.5 \pm 0.53$ & $10.0 \pm 0.00$ & $8.1 \pm 0.57$ & $8.56 \pm 0.35$ \\
Expert 4 & $8.9 \pm 0.32$ & $8.8 \pm 0.42$ & $8.9 \pm 0.30$ & $10.0 \pm 0.00$ & $8.8 \pm 0.42$ & $9.08 \pm 0.20$ \\
Expert 5 & $8.1 \pm 0.57$ & $7.6 \pm 0.52$ & $8.3 \pm 0.48$ & $10.0 \pm 0.00$ & $7.8 \pm 0.63$ & $8.36 \pm 0.31$ \\
\bottomrule
\end{tabular}
\end{table}

\subsubsection{Perplexity Sonar as Judge}
\begin{table}[H]
\centering
\caption{Perplexity Sonar Evaluating Domain Expert Rationales (Anonymized Condition)}
\small
\begin{tabular}{lcccccc}
\toprule
\textbf{Expert} & \textbf{Accuracy} & \textbf{Quality} & \textbf{Consistency} & \textbf{Completeness} & \textbf{Practicality} & \textbf{Final Score} \\
\midrule
Expert 1 & $8.7 \pm 0.4$ & $8.3 \pm 0.5$ & $9.0 \pm 0.3$ & $9.2 \pm 0.2$ & $8.5 \pm 0.4$ & $8.7 \pm 0.3$ \\
Expert 2 & $8.5 \pm 0.4$ & $8.0 \pm 0.5$ & $8.8 \pm 0.3$ & $8.9 \pm 0.3$ & $8.3 \pm 0.4$ & $8.5 \pm 0.3$ \\
Expert 3 & $8.6 \pm 0.4$ & $8.1 \pm 0.5$ & $8.9 \pm 0.3$ & $9.0 \pm 0.3$ & $8.4 \pm 0.4$ & $8.6 \pm 0.3$ \\
Expert 4 & $9.1 \pm 0.3$ & $8.9 \pm 0.4$ & $9.2 \pm 0.2$ & $9.5 \pm 0.2$ & $9.0 \pm 0.3$ & $9.1 \pm 0.2$ \\
Expert 5 & $8.4 \pm 0.4$ & $8.0 \pm 0.5$ & $8.7 \pm 0.3$ & $8.8 \pm 0.3$ & $8.2 \pm 0.4$ & $8.4 \pm 0.3$ \\
\bottomrule
\end{tabular}
\end{table}

\textbf{Note:} Expert identities correspond to the LLM mapping: Expert 1 = GPT-5.1, Expert 2 = Gemini 2.5 Pro, Expert 3 = Grok 4, Expert 4 = Claude 4.5 Sonnet, Expert 5 = Perplexity Sonar. This mapping was concealed from judges during the anonymized evaluation condition.

\section{Mathematical Proofs}
\label{app:proofs}

This appendix provides complete proofs and illustrations of the theoretical
properties introduced in Section~\ref{sec:theory}.  
All notation follows the definitions in Section~\ref{sec:definition}.

\subsection*{F.1 Proof of Zero-Sum Property}

\begin{proposition*}[Zero-Sum Property]
For any entity $j$,
\[
\sum_{i=1}^{n} \text{Bias}_A(i,j) = 0,
\qquad
\sum_{i=1}^{n} \text{Bias}_B(i,j) = 0.
\]
\end{proposition*}

\begin{proof}
We prove the attributed case; the anonymized case is identical.

By definition:
\[
\text{Bias}_A(i,j)
=
\text{Score}_{\text{judge}=i}(\text{LLM}=j)
-
\text{MeanScore}_{k \neq i}(\text{LLM}=j).
\]

Summing over all judges:
\[
\sum_{i=1}^n \text{Bias}_A(i,j)
=
\sum_{i=1}^n \text{Score}_{\text{judge}=i}(\text{LLM}=j)
-
\sum_{i=1}^n \text{MeanScore}_{k \neq i}(\text{LLM}=j).
\]

The first term expands directly:
\[
\sum_{i=1}^n \text{Score}_{\text{judge}=i}(\text{LLM}=j)
=
\sum_{k=1}^n s_{kj}.
\]

For the second term, note that each $s_{kj}$ (for fixed $k$) appears in 
$\sum_{k\neq i} s_{kj}$ for all $i \neq k$, i.e., exactly $(n-1)$ times.  
Thus:
\[
\sum_{i=1}^{n}
\sum_{k \neq i} s_{kj}
=
(n-1)\sum_{k=1}^{n} s_{kj}.
\]

Therefore:
\[
\sum_{i=1}^n
\text{MeanScore}_{k \neq i}(\text{LLM}=j)
=
\frac{1}{n-1}
\,(n-1)
\sum_{k=1}^{n} s_{kj}
=
\sum_{k=1}^{n} s_{kj}.
\]

Subtracting the two expressions gives:
\[
\sum_{i=1}^n \text{Bias}_A(i,j) = 0.
\]
\end{proof}

\textbf{Implication.}  
Bias is inherently relative: over-scoring by some judges necessarily implies 
under-scoring by others.

\subsection*{F.2 Proof of Self-Exclusion Preventing Circularity}

\begin{proposition*}[Self-Exclusion Prevents Circularity]
Because consensus excludes judge $i$,
\[
\frac{\partial \,\text{Bias}_A(i,j)}
     {\partial \,\text{Score}_{\text{judge}=i}(\text{LLM}=j)} = 1.
\]
\end{proposition*}

\begin{proof}
By definition:
\[
\text{Bias}_A(i,j)
=
\text{Score}_{\text{judge}=i}(\text{LLM}=j)
-
\text{MeanScore}_{k \neq i}(\text{LLM}=j).
\]

Since the consensus sum excludes $i$:
\[
\frac{\partial \,\text{MeanScore}_{k \neq i}}
     {\partial \,\text{Score}_{\text{judge}=i}(\text{LLM}=j)} = 0.
\]

Therefore:
\[
\frac{\partial \text{Bias}_A(i,j)}
     {\partial \,\text{Score}_{\text{judge}=i}(\text{LLM}=j)}
=
1 - 0
=
1.
\]
\end{proof}

\textbf{Implication.}  
Deviation is measured against an independent reference.  
A judge cannot affect the benchmark used to evaluate its own behavior.

\subsection*{F.3 Contrast with Non-Self-Excluding Consensus}

If judge $i$ were included in the consensus, circularity would arise.

\paragraph{Naive consensus.}
\[
\text{Consensus}_{\text{naive}}(j)
=
\frac{1}{n}
\sum_{k=1}^{n}
\text{Score}_{\text{judge}=k}(\text{LLM}=j).
\]

Then
\[
\frac{\partial \text{Consensus}_{\text{naive}}(j)}
{\partial \,\text{Score}_{\text{judge}=i}} 
=
\frac{1}{n}
\neq 0.
\]

\paragraph{Naive deviation.}
\[
\text{Bias}_{\text{naive}}(i,j)
=
\text{Score}_{\text{judge}=i}
-
\text{Consensus}_{\text{naive}}(j).
\]

Differentiating:
\[
\frac{\partial \text{Bias}_{\text{naive}}(i,j)}
     {\partial \text{Score}_{\text{judge}=i}}
=
1 - \frac{1}{n}
=
\frac{n-1}{n}.
\]

\textbf{Interpretation.}  
This $\tfrac{n-1}{n}$ factor systematically underestimates true deviation.  
For $n=5$, deviations shrink by 20\%, obscuring true evaluation patterns.  

\paragraph{Illustration.}

Assume two judges score entity $j$:
\[
s_{1j}=10.0, \qquad s_{2j}=8.0.
\]

\textit{With self-exclusion:}
\[
\text{Bias}_A(1,j) = 10.0 - 8.0 = +2.0,
\qquad
\text{Bias}_A(2,j) = 8.0 - 10.0 = -2.0.
\]

\textit{Without self-exclusion:}
\[
\text{Consensus}_{\text{naive}}(j) = 9.0,
\qquad
\text{Bias}_{\text{naive}}(1,j)=+1.0,
\quad
\text{Bias}_{\text{naive}}(2,j)=-1.0.
\]

Judge 1’s true deviation is +2.0, but the naive metric reports +1.0 —  
a 50\% underestimation caused by including $s_{1j}$ in the consensus reference.

\paragraph{General case.}
For any judge deviation $\Delta$ from peer consensus:
\[
\text{Bias}_{\text{naive}} = \frac{n-1}{n}\,\Delta.
\]

\textbf{Self-exclusion removes this circularity entirely}, ensuring unbiased deviation measurement.

\bigskip
These results confirm the two structural properties that underpin the consensus-deviation metric:
(1) bias is relative and calibrated (zero-sum), and  
(2) deviation is measured independently of the judge being evaluated (no circularity).

\end{document}